\documentclass[sigconf]{acmart}

\copyrightyear{2026}
\acmYear{2026}

\setcopyright{cc}
\setcctype{by}

\acmConference[MM '26]
{Proceedings of the 34th ACM International Conference on Multimedia}
{November 10--14, 2026}
{Rio de Janeiro, Brazil}

\acmBooktitle{Proceedings of the 34th ACM International Conference on Multimedia (MM '26), November 10--14, 2026, Rio de Janeiro, Brazil}

\acmDOI{10.1145/3767308.3836463}

\acmISBN{979-8-4007-2213-4/2026/11}

\usepackage{hyperref}
\usepackage{multirow}
\usepackage{threeparttable}
\usepackage{algorithm}
\usepackage{algpseudocode}
\usepackage{xcolor}
\usepackage{enumitem}
\usepackage{graphicx}
\usepackage{subfig}
\usepackage{tikz}
\usepackage{booktabs}
\usepackage{amsthm}
\usepackage[most]{tcolorbox}

\newtheorem{proposition}{Proposition}
\newtheorem{corollary}{Corollary}

\usetikzlibrary{positioning,fit}

\renewcommand{\shortauthors}{Zhan et al.}

\definecolor{indigo}{HTML}{4B0082}
\definecolor{demolinkcolor}{HTML}{0B5FA5}
\hypersetup{
  colorlinks=true,
  urlcolor=demolinkcolor,
  linkcolor=demolinkcolor,
  citecolor=demolinkcolor,
}

\begin{document}

\setlength{\textfloatsep}{8pt plus 2pt minus 2pt}
\setlength{\floatsep}{8pt plus 2pt minus 2pt}
\setlength{\intextsep}{8pt plus 2pt minus 2pt}

\title{Seeing is Free, Speaking is Not: Uncovering the True Energy Bottleneck in Edge VLM Inference}

\author{Junfei Zhan}
\affiliation{%
  \department{Department of Computing}
  \institution{Imperial College London}
  \city{London}
  \country{United Kingdom}
}
\email{j.zhan26@imperial.ac.uk}

\author{Haoxun Shen}
\affiliation{%
  \department{Department of Electrical and Systems Engineering}
  \institution{University of Pennsylvania}
  \city{Philadelphia}
  \state{PA}
  \country{United States}
}
\email{haoxun@engineering.upenn.edu}

\author{Mingang Guo}
\affiliation{%
  \department{Department of Electrical and Systems Engineering}
  \institution{University of Pennsylvania}
  \city{Philadelphia}
  \state{PA}
  \country{United States}
}
\email{gmingang@engineering.upenn.edu}

\author{Zixuan Huang}
\affiliation{%
  \institution{Xiaohongshu Inc}
  \city{Shanghai}
  \country{China}
}
\email{huangzixuan1@xiaohongshu.com}

\author{Tengjiao He}
\authornote{Corresponding author.}
\affiliation{%
  \department{College of Information Science and Technology}
  \institution{Jinan University}
  \city{Guangzhou}
  \country{China}
}
\email{htj2018@jnu.edu.cn}


\begin{abstract}
Vision-Language Models (VLMs) are the perceptual backbone of embodied AI, but their energy footprint on edge hardware remains poorly understood. Existing efficiency efforts focus predominantly on reducing visual tokens, implicitly treating visual processing as the dominant energy cost.
We overturn this implicit assumption through the first systematic energy profiling of on-device VLM inference, spanning five models across three architecture families, four input resolutions, and two hardware platforms (NVIDIA RTX 3070 and Jetson Orin NX).
Our analysis yields three findings.
First, average inference power is a model-intrinsic constant, invariant to input resolution, image complexity, and prompt type, with less than 5\% variation across all conditions. This means that all energy variation across inputs must arise from variation in inference time, not from variation in power draw.
Second, each output token costs 11 to 39$\times$ more wall-clock time than each input token due to the compute-bound and memory-bound asymmetry between prefill and decode, making output token count the dominant driver of both latency and energy.
%
Third, image complexity, measured by the number of objects in an image, induces up to 4.1$\times$ energy differences at identical resolution. This variation arises not from increased visual processing cost, but from differences in output length.
These findings expose a fundamental limitation of visual token pruning: even removing all visual tokens saves at most 10\% of total energy for fixed-token models. Across models spanning 1 billion to 8 billion parameters, controlling output length saves up to 97\% of total energy, with the energy dominance of decoding growing stronger at larger model scale. In short, the true energy bottleneck in edge VLM inference is not what the model sees, but how much it says.
Code is available at \url{https://github.com/Junfei-Z/seeing-is-free}.
\end{abstract}

\ccsdesc[500]{Computer systems organization~Embedded and cyber-physical systems}
\ccsdesc[500]{Computing methodologies~Machine learning}
\ccsdesc[300]{Hardware~Power and energy}

\keywords{Vision-Language Models, Energy Profiling, Edge Inference, 
On-Device AI, Embodied AI, Power Measurement}

\maketitle


%
\begin{figure}[t!]
    \centering
    \subfloat[Energy Decomposition]{\includegraphics[width=0.48\linewidth]{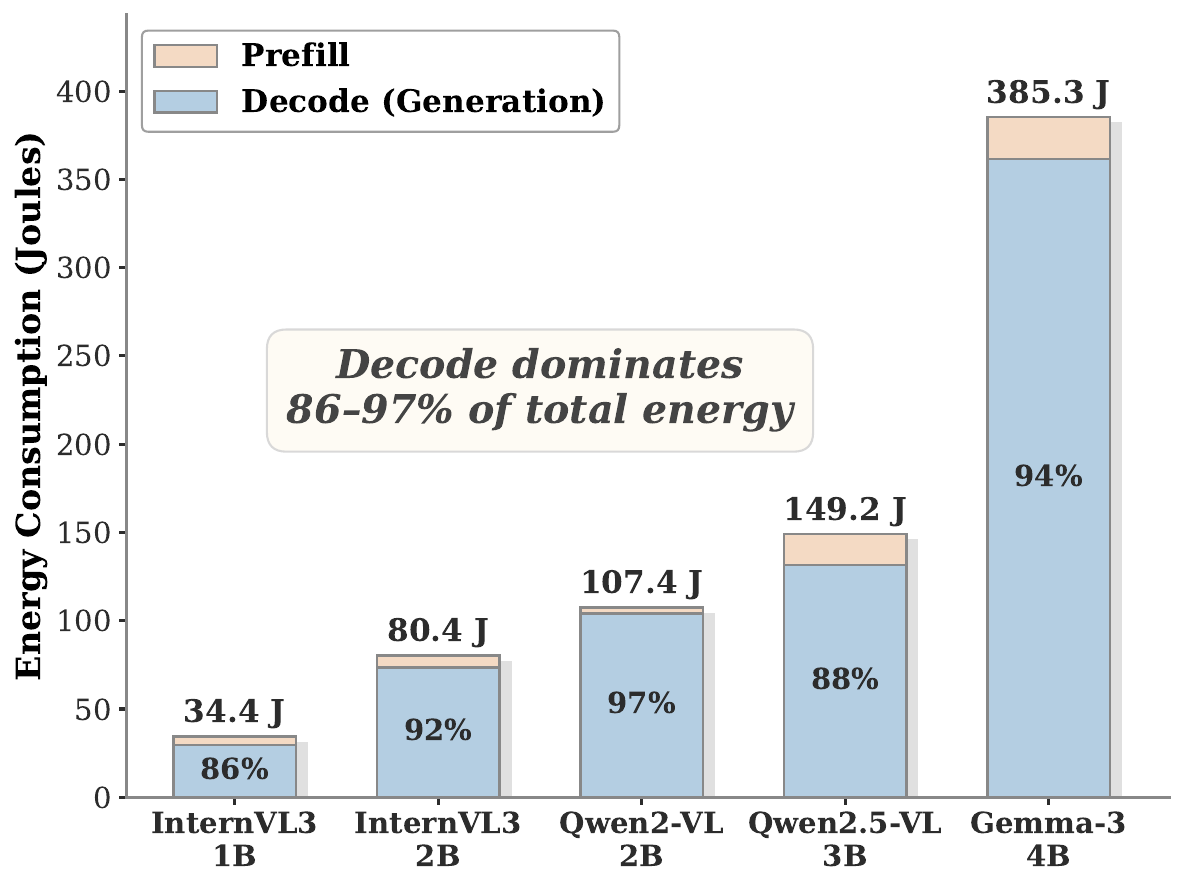}}
    \hfill 
    \subfloat[Per-token Cost Ratio]{\includegraphics[width=0.48\linewidth]{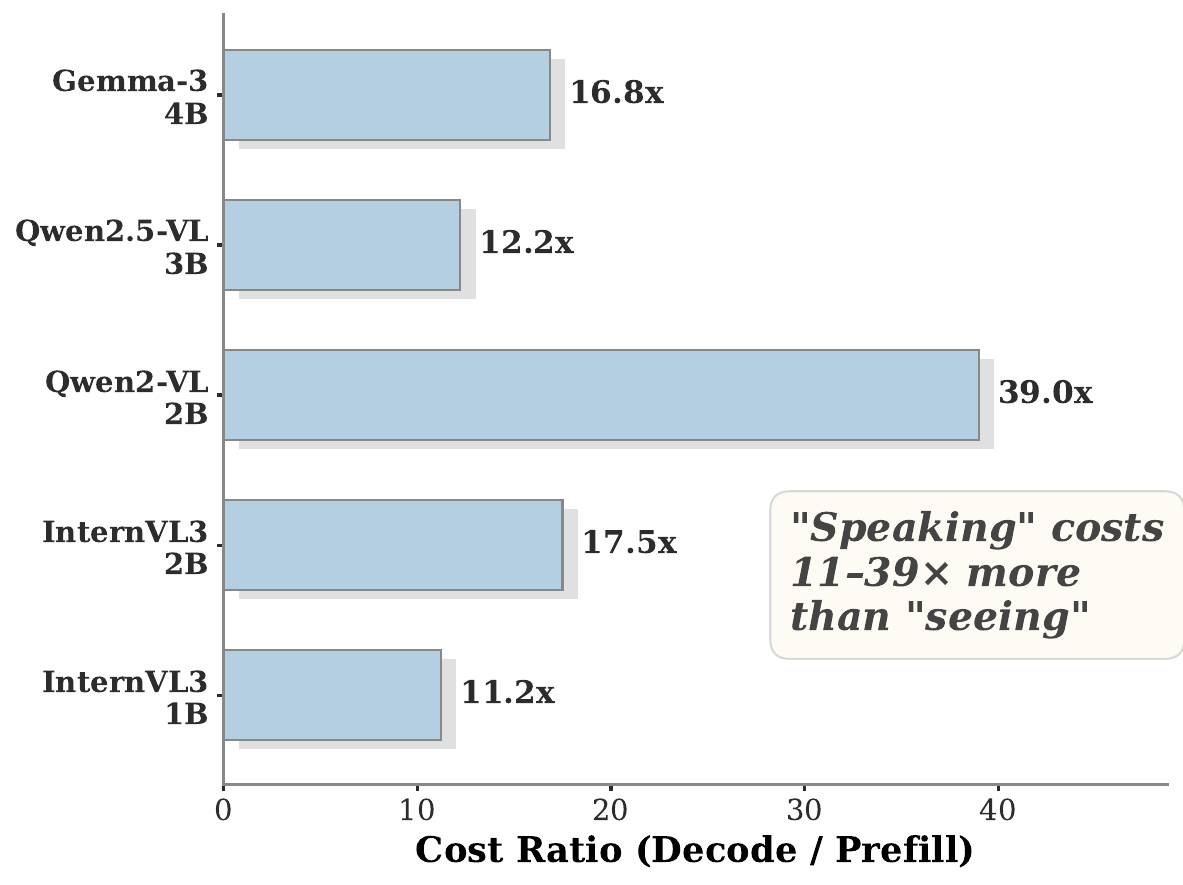}}
    \caption{Energy anatomy of VLM inference on the RTX~3070 at 448$\times$448. 
(a)~Decode dominates at 86--97\%. 
(b)~Each output token costs 11--39$\times$ more than each input token.}
    \label{fig:intro} 
\end{figure}

\section{Introduction}

Embodied AI agents are increasingly deployed in real-world tasks like robotic manipulation and drone inspection, demanding real-time visual reasoning and precise action execution. 
The Vision-Language-Action (VLA) paradigm~\cite{ma2024vlasurvey} currently dominates this landscape by employing Vision-Language Models (VLMs) as the central processing unit. 
Key milestones such as RT-2~\cite{brohan2023rt2} and PaLM-E~\cite{driess2023palme} illustrate the power of cross-modal knowledge transfer, while recent works like $\pi_0$~\cite{black2024pi0} and OpenVLA~\cite{kim2024openvla} demonstrate that even a 7B VLM backbone can surpass the 55B RT-2-X by 16.5\%, highlighting that compact VLMs already serve as effective cores for embodied agents.



%
Deploying these embodied agents requires running VLM inference directly on edge devices such as NVIDIA Jetson modules, mobile SoCs, and embedded GPUs for AR glasses, as cloud offloading is often impractical due to stringent latency, connectivity, and privacy constraints~\cite{abstreiter2025painful}.
To enable on-device inference, the community has developed a series of efficient VLMs for edge devices. 
For instance, MobileVLM~\cite{chu2024mobilevlm} achieves 21.5 tokens/sec on mobile SoCs, while FastVLM~\cite{vasu2025fastvlm} delivers an $85\times$ speedup in time-to-first-token compared to LLaVA-OneVision. SmolVLM~\cite{marafioti2025smolvlm} further reduces the model size to 256M parameters, requiring less than 1\,GB of GPU memory. However, even these efficient models must operate within strict hardware budgets. Edge batteries typically provide only 40 to 100\,Wh, whereas a single VLM query can consume between 50 and 500\,Joules~\cite{arya2025understanding}. 
As inference queries accumulate over continuous perception-action cycles, total energy consumption can drain the battery within hours of operation. Consequently, energy consumption is not merely a secondary optimization target but a binding constraint that dictates the operational capacity of embodied agents.

%
Despite the growing importance of energy efficiency for on-device VLMs, no existing work has systematically measured the energy consumption of multimodal VLM inference on edge hardware.
While energy profiling has been studied for text-only Large Language Models (LLMs)~\cite{xu2025camel,husom2025sustainable,jain2025clear}, these works do not extend to multimodal settings, leaving the energy impact of vision encoding uncharacterized. Extending these methods to VLMs is not straightforward, as multimodal inference interleaves vision encoding, prefill, and decode phases with distinct computational characteristics, making it difficult to attribute energy to individual stages~\cite{pope2023efficiently}.
Without such measurements, existing VLM efficiency works~\cite{chen2024fastv,xing2025pyramiddrop,yang2025visionzip} optimize visual token reduction based on computational cost alone, implicitly assuming that visual token processing is the dominant energy cost. This assumption has never been empirically verified. 
If it does not hold, then optimizing visual tokens may yield far less energy savings than expected.

We directly test the implicit assumption that visual token processing is the primary energy burden in multimodal VLM inference, asking: \textit{where do the Joules go in on-device VLM inference, and how can this understanding enable energy-aware deployment?} 
Through systematic energy profiling of five VLMs across three architecture families, four input resolutions, and two hardware platforms (NVIDIA RTX 3070 and Jetson Orin NX), we uncover a counterintuitive result: \emph{seeing is free, speaking is not}. As illustrated in Figure~\ref{fig:intro}, each output token costs 11 to 39$\times$ more energy than each input token including visual tokens, making output length the dominant energy variable. Across all models, output token count nearly perfectly predicts per-image energy, while image complexity can cause up to $4.1\times$ energy differences at identical resolution through its effect on output length. Meanwhile, the effect of input resolution on energy is comparatively modest: dynamic-token architectures such as Qwen2-VL~\cite{wang2024qwen2vl,bai2025qwen25vl} show 8 to 24\% energy increase from 224$\times$224 to 896$\times$896, whereas fixed-token architectures such as Gemma-3~\cite{gemmateam2025gemma3} and InternVL3~\cite{zhu2025internvl3} remain virtually unchanged. 
These findings reveal that the true energy bottleneck in on-device VLM inference is not \emph{what the model sees}, but \emph{how much it says}. Specifically, our contributions are as follows:


\begin{itemize}[leftmargin=*,nosep]
    \item \textbf{Power is a model fingerprint.} We show that average inference power draw is a model-intrinsic constant, invariant to input resolution, image complexity, and prompt type, with less than 5\% variation across all conditions. This reduces energy analysis to understanding what determines inference time.
    %
    %
    \item \textbf{Energy decomposition.} We decompose inference energy into prefill and decode phases and show that autoregressive decoding accounts for 86 to 97\% of total energy, driven by the compute-bound and memory-bound asymmetry between the two phases.
    \item \textbf{Visual token pruning upper bound.} We derive a theoretical upper bound showing that even removing all visual tokens saves at most 10\% of total energy for fixed-token models, while controlling output length saves up to 97\%.
    \item \textbf{Cross-model energy predictor.} We show that a linear model with five features, namely model size, input token count, output token count, and two interaction terms, explains 98.6\% of energy variance across all configurations without per-model calibration, validating that VLM inference energy is low-dimensional and architecturally determined.

\end{itemize}


\section{Preliminaries: VLM Inference on Edge}
\label{sec:prelim}

We review the inference pipeline of VLMs to establish the terminology and computational stages used throughout this paper (Figure~\ref{fig:pipeline}).

\begin{figure*}
    \centering
    \small
    \includegraphics[width=\linewidth]{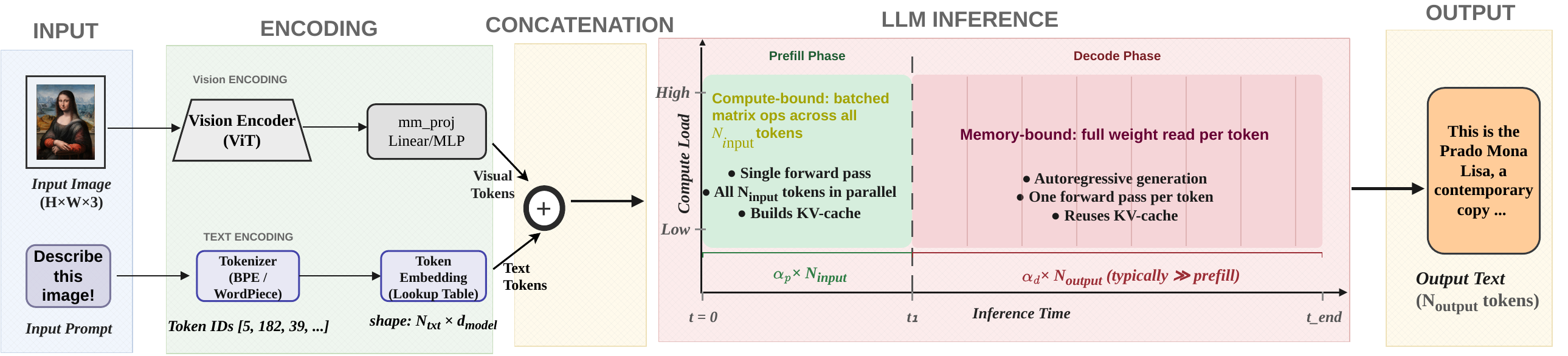}
    \caption{VLM inference pipeline. An input image and text prompt are encoded separately into visual and text tokens, concatenated into a unified sequence, and processed by the language model in two phases: a parallel prefill followed by sequential autoregressive decoding.}
    \label{fig:pipeline}
\end{figure*}

\subsection{Input Encoding and Visual Token Strategies}
\label{sec:prelim:encoding}
A VLM receives an image of height $H$ and width $W$ along with a text prompt.
On the vision side, a \emph{vision encoder}, typically a Vision Transformer (ViT)~\cite{dosovitskiy2021vit}, partitions the input image into fixed-size patches of $p \times p$ pixels and maps each patch into an embedding vector.
A \emph{modality projector} (\texttt{mm\_proj}), commonly a linear layer or shallow MLP, then maps these embeddings into the language model's token space, producing $N_{\text{vis}}$ \emph{visual tokens}.
On the text side, a subword tokenizer converts the prompt into $N_{\text{text}}$ \emph{text tokens}, each mapped to a $d_{\text{model}}$-dimensional embedding.

The number of visual tokens $N_{\text{vis}}$ depends on the model's \emph{visual token strategy}, which defines how input resolution maps to token count.
We distinguish two strategies among contemporary architectures.
The first is \textbf{fixed-token architectures}.
Models such as InternVL3~\cite{zhu2025internvl3} and Gemma-3~\cite{gemmateam2025gemma3} resize all images to a predefined internal resolution before encoding, producing a constant number of visual tokens regardless of input resolution:
\begin{equation}
\label{eq:fixed-token}
N_{\text{vis}} = \left\lfloor\frac{R_{\text{internal}}}{p}\right\rfloor^2,
\end{equation}
where $R_{\text{internal}}$ is the fixed internal resolution and $p$ is the ViT patch size.
For example, InternVL3 uses a 490-pixel internal resolution with patch size 14, yielding $(490/14)^2 = 1{,}225$ patch embeddings that are downsampled to 265 tokens via a pixel-shuffle projector. This means that whether the original image is $224\times224$ or $896\times896$, the vision encoder always produces the same number of tokens, and the additional resolution information is discarded during resizing.

The second is \textbf{dynamic-resolution architectures}.
Models such as Qwen2-VL~\cite{wang2024qwen2vl} and Qwen2.5-VL~\cite{bai2025qwen25vl} process images at or near their native resolution, so $N_{\text{vis}}$ scales with the input image area can be approximated as:
\begin{equation}
\label{eq:dynamic-token}
N_{\text{vis}} = \left\lfloor\frac{H}{p \cdot s}\right\rfloor \times \left\lfloor\frac{W}{p \cdot s}\right\rfloor,
\end{equation}
where $s$ is a spatial merge factor that groups neighboring patches into a single token before feeding them to the language model. For Qwen2-VL, $s=2$, meaning every $2\times2$ block of patches is merged into one visual token.
In practice, $N_{\text{vis}}$ ranges from 73 tokens at $224\times224$ to over 1{,}000 tokens at $896\times896$. This resolution-dependent scaling has direct implications for inference cost, as we quantify in Section~\ref{sec:time}.

\subsection{Two-Phase LLM Inference}
\label{sec:prelim:inference}
The visual and text token sequences are concatenated, together with any system prompt tokens $N_{\text{sys}}$, to form a unified input sequence of length $N_{\text{in}} = N_{\text{sys}} + N_{\text{vis}} + N_{\text{text}}$.
The language model processes this sequence in two computational phases~\cite{pope2023efficiently}.

During \emph{prefill}, all $N_{\text{in}}$ tokens are processed in a single parallel forward pass through the transformer.
Because all positions are known in advance, the computation performs large matrix multiplications over the entire input sequence, achieving high \emph{arithmetic intensity} (FLOPs per byte of memory traffic) that fully saturates the GPU's compute units.
The prefill phase populates the key-value (KV) cache, which stores the intermediate representations of all input tokens for reuse during decoding.

During \emph{decode}, output tokens are generated one at a time in an autoregressive loop.
At each step, the model reads the full set of weights from memory but performs computation for only a single token, yielding low arithmetic intensity that is bottlenecked by \emph{memory bandwidth} rather than compute.
The newly generated token is appended to the KV cache, and the process repeats until the model emits an end-of-sequence token or reaches the maximum output length.
Total inference time thus decomposes as:
\begin{equation}
\label{eq:time-total}
t = t_{\text{prefill}}(N_{\text{in}}) + t_{\text{decode}}(N_{\text{out}}),
\end{equation}
where the per-token cost of decode far exceeds that of prefill due to this compute-bound and memory-bound dichotomy, a point we quantify in Section~\ref{sec:time}.

\section{Power is a Model's Fingerprint}
\label{sec:power}

We begin by asking what drives the energy cost of a single VLM inference.
The total energy $E$ of a single inference is the product of two quantities: average power $\bar{P}$ and inference time $t$, giving $E = \bar{P} \times t$.
We first investigate the power term by measuring the per-inference average power, defined as the total energy of a single inference run divided by its wall-clock duration.

We conduct experiments on two edge hardware platforms: an NVIDIA RTX~3070 Laptop GPU (8\,GB VRAM, frequency locked at 1500\,MHz) and an NVIDIA Jetson Orin NX (15\,W power mode with clocks locked via \textit{jetson\_clocks}).
All models are served via \textit{llama-server} (llama.cpp) with greedy decoding ($T{=}0$).
Per-inference average power is computed from power samples collected at 100\,ms intervals: via HWiNFO64 Pro (total system power) on the laptop, and via the onboard INA3221 sensor (VDD\_IN rail) on the Jetson.
As summarized in Table~\ref{tab:models}, we profile five VLMs spanning three architectural families, 1B--4B parameters, 
and both fixed-token and dynamic-resolution visual token strategies.

\begin{table}[t]
\centering
\caption{Profiled VLMs and their visual token strategies.}
\label{tab:models}
\resizebox{\linewidth}{!}{%
\begin{tabular}{llccccc}
\toprule
\textbf{Model} & \textbf{Vision Encoder} & \textbf{Params} & \textbf{Quant} & \textbf{Token Strategy} & \textbf{Vis Tokens} \\
\midrule
InternVL3-1B  & InternViT & 1B & Q8\_0    & Fixed   & 265 / 265 / 265 / 265    \\
InternVL3-2B  & InternViT & 2B & Q4\_K\_M & Fixed   & 265 / 265 / 265 / 265     \\
Qwen2-VL-2B   & Qwen-ViT  & 2B & Q4\_K\_M & Dynamic & 73 / 265 / 585 / 1033    \\
Qwen2.5-VL-3B & Qwen-ViT  & 3B & Q4\_K\_M & Dynamic & 73 / 265 / 585 / 1033    \\
Gemma-3-4B    & SigLIP    & 4B & Q4\_K\_M & Fixed   & 265 / 265 / 265 / 265   \\
\bottomrule
\end{tabular}%
}
\end{table}

We systematically vary three factors that practitioners might expect to influence power draw: input resolution, image content complexity, and task prompt.
Surprisingly, none of them does.

\subsection{Power is Resolution-Invariant}
\label{sec:power_resolution}

We first examine whether input resolution affects per-inference average power.
On the RTX~3070, each of the five models listed in Table~\ref{tab:models} is evaluated across four resolutions ($224^2$ to $896^2$) over 12 runs per configuration using COCO~2017 validation images~\cite{lin2014microsoft}, with the first two runs discarded as warmup.

As shown in Figure~\ref{fig:power_resolution}(a), within every model the power distributions across resolutions are nearly identical: InternVL3-1B draws $52.4 \pm 1.6$\,W, Qwen2.5-VL-3B draws $76.8 \pm 3.2$\,W, and Gemma-3-4B draws $90.5 \pm 0.5$\,W, with per-model coefficients of variation below 5\%.
This holds for both fixed-token and dynamic-resolution models (see Table~\ref{tab:models}).
Notably, the 14$\times$ increase in visual tokens for Qwen2.5-VL-3B from $224^2$ to $896^2$ produces no measurable change in average power.

Figure~\ref{fig:power_resolution}(b) further reveals that this resolution-invariant power scales linearly with model size.
A linear regression across all five models yields:
\begin{equation}
\bar{P} = 12.1 \, S + 42.2 \quad (R^2 = 0.918)
\label{eq:power}
\end{equation}
where $S$ is the number of parameters in billions.
Each additional billion parameters adds approximately 12\,W to the per-inference average power.

\begin{figure}[t]
    \centering
    \begin{minipage}[t]{0.48\linewidth}
        \centering
        \includegraphics[width=\linewidth]{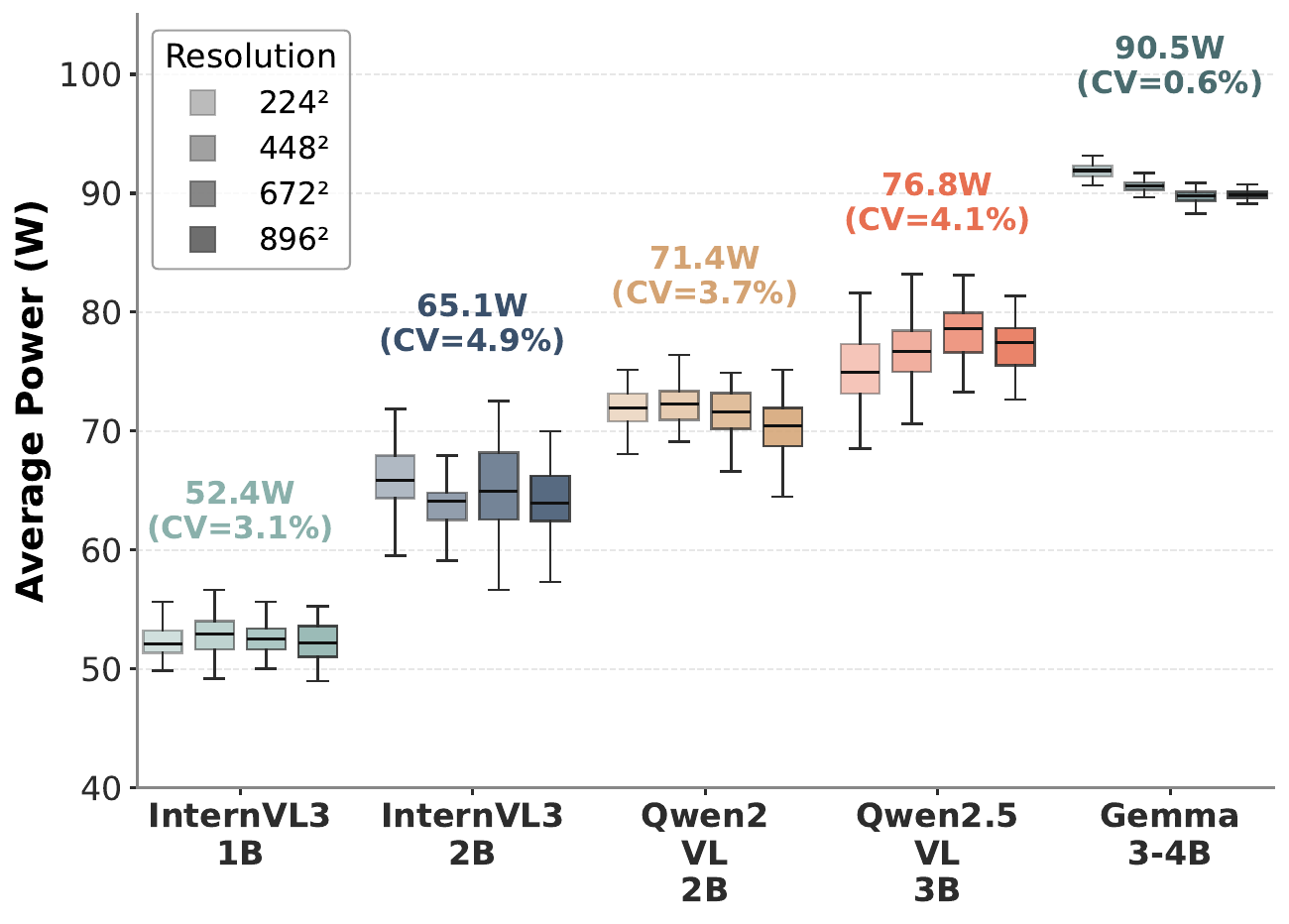}
        \vspace{-6mm}
        \caption*{(a) Power vs. resolutions}
    \end{minipage}
    \hfill
    \begin{minipage}[t]{0.48\linewidth}
        \centering
        \includegraphics[width=\linewidth]{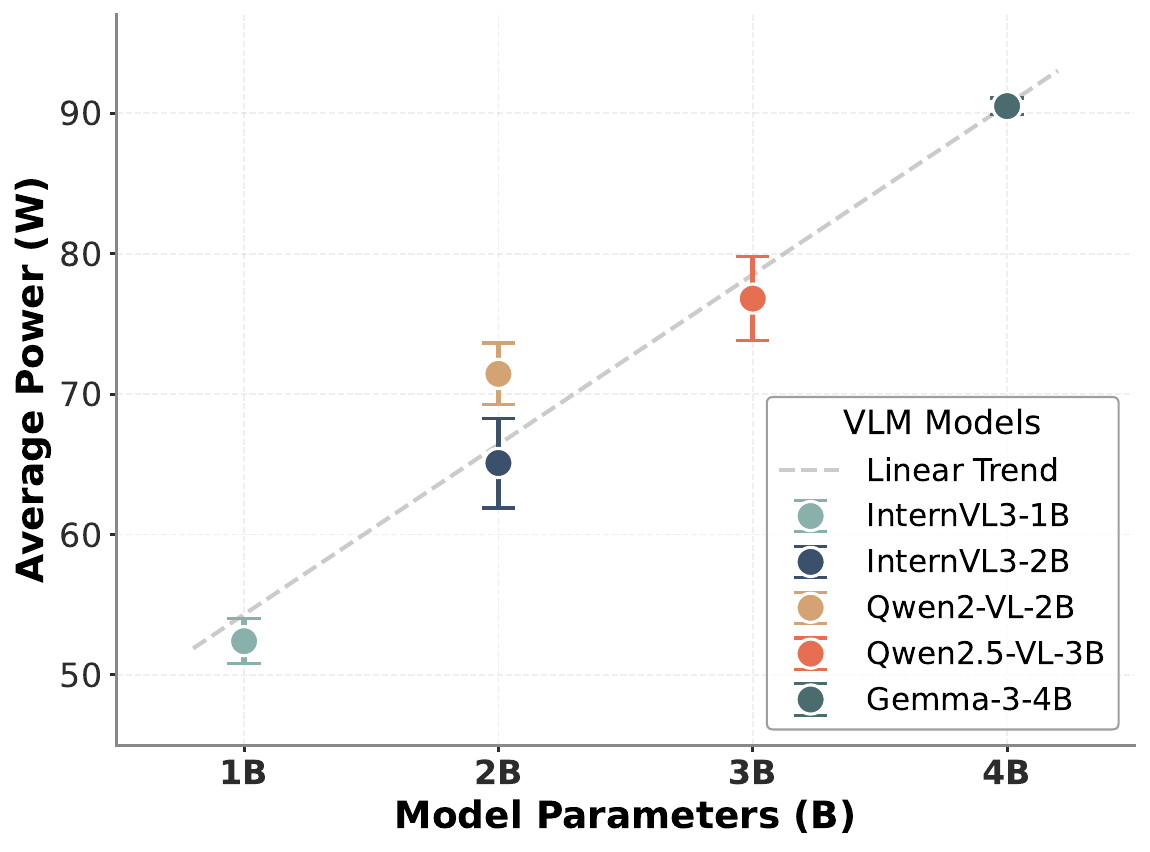}
        \vspace{-6mm}
        \caption*{(b) Power vs.\ model size}
    \end{minipage}
    \vspace{1mm}
    \caption{Per-inference average power across resolutions and model sizes.}
    \label{fig:power_resolution}
\end{figure}

\subsection{Power is Content-Invariant}
\label{sec:power_content}
Resolution invariance might be expected for fixed-token models, since their computational graph does not change with resolution.
A stronger test is whether image content, which directly affects what the model ``sees'' and how it responds, can modulate the per-inference average power.

To test this, we construct a dataset of 40 COCO~2017 images spanning six content complexity tiers based on annotated object counts: T1 (single object, minimal scene) through T6 (12--16 objects, complex scene).
We profile InternVL3-1B and Qwen2.5-VL-3B on a Jetson Orin NX at $448^2$ resolution.
As shown in Figure~\ref{fig:power_jetson}(a), despite a 14$\times$ increase in object count from T1 to T6, and despite these images eliciting vastly different output lengths (from 95 to 482 tokens on average), the per-inference average power remains flat: 14.6\,W (CV\,=\,1.1\%) for InternVL3-1B and 14.1\,W (CV\,=\,2.9\%) for Qwen2.5-VL-3B.
This result is noteworthy because, as we will show in Section~\ref{sec:time}, the \emph{energy} consumed by these images varies by up to 4$\times$.
However, all of that variation comes from differences in inference time, not from differences in power.

\subsection{Power is Prompt-Invariant}
\label{sec:power_prompt}
 
We test whether the task prompt affects power.
Since the prompt controls output length and thus the balance between prefill and decode computation, we design two prompts that elicit maximally different output lengths while keeping the input image identical.
On the same 40 images at $448^2$, we compare:
(I) \textbf{Describe}: ``Describe this image.'' This open-ended prompt elicits detailed captions (avg.\ 398 output tokens for InternVL3-1B).
(II) \textbf{Short}: ``What is the main object in this image? Answer in one word.'' This constrained prompt terminates almost immediately after prefill (avg.\ 3 output tokens, including the EOS token).
 
As shown in Figure~\ref{fig:power_jetson}(b), despite this roughly 100$\times$ difference in output length, the per-inference average power is indistinguishable between the two prompts: 14.6\,W vs.\ 14.8\,W for InternVL3-1B (CV\,=\,1.4\%), and 14.1\,W vs.\ 14.4\,W for Qwen2.5-VL-3B (CV\,=\,3.4\%).

\begin{figure*}[t]
    \centering
    \begin{minipage}[t]{0.48\linewidth}
        \centering
        \includegraphics[width=\linewidth]{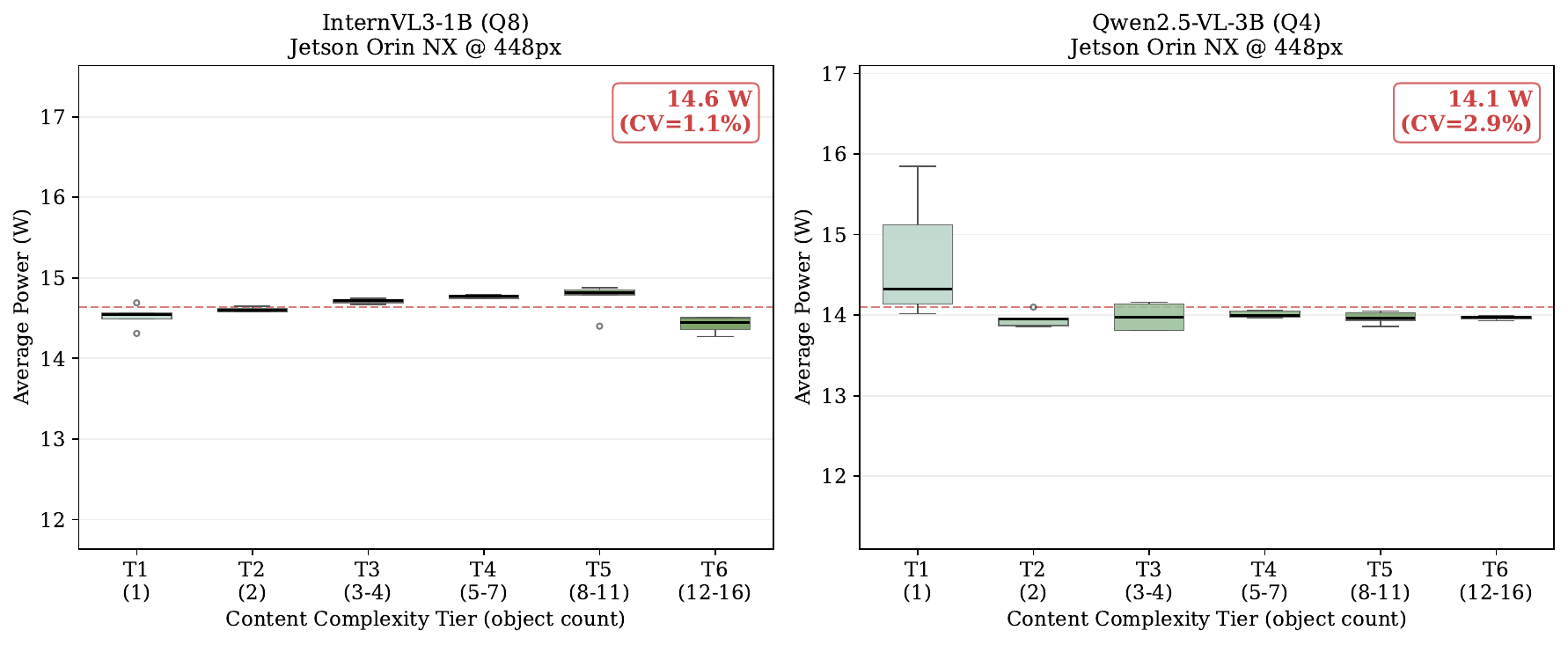}
        \vspace{-6mm}
        \caption*{(a) content complexity}
    \end{minipage}
    \hfill
    \begin{minipage}[t]{0.48\linewidth}
        \centering
        \includegraphics[width=\linewidth]{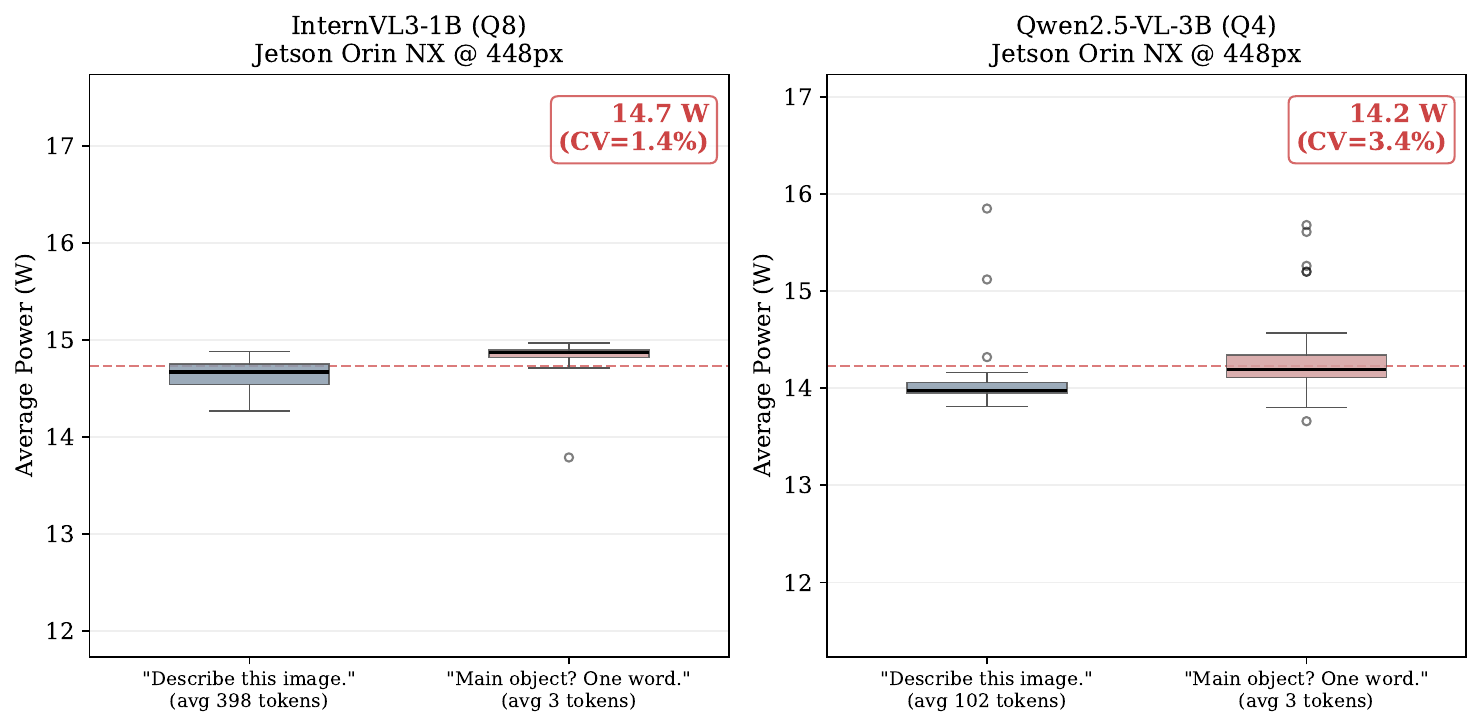}
        \vspace{-6mm}
        \caption*{(b) prompts}
    \end{minipage}
    \vspace{1mm}
    \caption{Power stability under varying input conditions.
(a) Image content complexity.
(b) Prompt type.}
    \label{fig:power_jetson}
\end{figure*}

\subsection{Implication: Energy $=$ Power $\times$ Time}
\label{sec:ept}

Section~\ref{sec:power_resolution}--\ref{sec:power_prompt} establish that per-inference average power is determined solely by the model's size and architecture, not by input resolution, image content, or task prompt.
We refer to this stable value as the model's \emph{power fingerprint}, denoted $\bar{P}$.
This finding is consistent across two hardware platforms (laptop GPU and embedded Jetson), five models spanning three architectural families, and a wide range of input conditions.

The underlying reason is that edge inference via llama.cpp is a sequential, token-by-token process: during prefill the GPU performs dense matrix multiplications over all input tokens, and during decode it performs the same operations one token at a time.
In both phases the GPU is fully saturated, so the instantaneous power draw remains constant regardless of what is being processed.
Changing the input image, resolution, or prompt alters \emph{how long} the GPU runs, but not \emph{how hard} it works at any given moment.

Since $\bar{P}$ is effectively constant for a given model--hardware pair, total inference energy simplifies to $E = \bar{P} \times t$, where $t$ is the total inference latency.
This has a direct consequence: \emph{all} energy variation across inputs, resolutions, and image content must arise from variation in inference time, not from variation in power draw.
The question ``where do the Joules go?'' therefore reduces to ``where does the time go?''

\section{Where Does the Time Go?}
\label{sec:time}

Section~\ref{sec:power} established that power draw is effectively constant during inference, reducing the energy equation to $E = \bar{P} \times t$.
The natural follow-up question is: what determines inference time $t$?
In this section, we decompose inference latency into its constituent phases and show that a single factor, the number of output tokens, overwhelmingly dominates wall-clock time across all models and hardware platforms we study.

\subsection{Two-Phase Decomposition}
\label{sec:time:decomp}

As reviewed in Section~\ref{sec:prelim}, autoregressive VLM inference proceeds in two phases: a compute-bound \emph{prefill} that processes all $N_{\text{in}}$ input tokens in a single parallel forward pass, and a memory-bandwidth-bound \emph{decode} that generates output tokens one at a time, each requiring a separate forward pass.
This architectural dichotomy yields a simple decomposition of wall-clock inference time:
\begin{equation}
\label{eq:time-decomp}
t_{\text{wall}} \;=\; t_{\text{prefill}}(N_{\text{in}}) \;+\; t_{\text{decode}}(N_{\text{out}}).
\end{equation}
Because prefill processes all input tokens in one pass while decode requires one pass per output token, the \emph{marginal cost} of an additional input token and an additional output token differ by at least an order of magnitude.
We formalize this asymmetry as follows.

\begin{definition}[Per-token marginal latency]
\label{def:marginal-cost}
Let $\alpha_p$ denote the marginal wall-clock cost of adding one input token to the prefill, and $\alpha_d$ the marginal wall-clock cost of generating one additional output token during decode:
\begin{equation}
\label{eq:alpha}
\alpha_p \;=\; \frac{\partial\, t_{\emph{prefill}}}{\partial\, N_{\emph{in}}}, \qquad
\alpha_d \;=\; \frac{\partial\, t_{\emph{decode}}}{\partial\, N_{\emph{out}}}.
\end{equation}
\end{definition}

Under these definitions, wall-clock time is well approximated by the linear latency model:
\begin{equation}
\label{eq:time-linear}
t_{\text{wall}} \;\approx\; \alpha_p \cdot N_{\text{in}} \;+\; \alpha_d \cdot N_{\text{out}} \;+\; \beta,
\end{equation}
where $\beta$ captures fixed overhead (model initialization, tokenization, sampling).
The key claim of this section is that $\alpha_d \gg \alpha_p$, which makes $N_{\text{out}}$ the dominant driver of inference time.

\subsection{The Decode-Prefill Cost Asymmetry}
\label{sec:time:asymmetry}

We estimate the parameters of Equation~\ref{eq:time-linear} empirically for all five models on the RTX~3070.
For models with dynamic visual tokenization (Qwen2-VL, Qwen2.5-VL), the number of input tokens $N_{\text{in}}$ varies across resolutions while $N_{\text{out}}$ varies across images, enabling direct multivariate regression.
For fixed-token models (InternVL3, Gemma-3), $N_{\text{in}}$ is constant across resolutions, so we estimate $\alpha_p$ by comparing the intercepts of latency-vs-$N_{\text{out}}$ regressions between multimodal and text-only conditions, which differ only in $N_{\text{in}}$.

Table~\ref{tab:alpha-beta} reports the fitted coefficients.
Across all models, the decode-to-prefill cost ratio $\alpha_d / \alpha_p$ ranges from 11$\times$ to 39$\times$, confirming the architectural prediction: each decode step performs a full weight-reading pass for a single token, while prefill amortizes the same pass across all input tokens simultaneously.

\begin{table}[t]
\centering
\caption{Fitted latency model parameters. The ratio $\alpha_d/\alpha_p$ quantifies the decode-prefill cost asymmetry.}
\label{tab:alpha-beta}
\small
\begin{tabular}{lccccr}
\toprule
\textbf{Model}  & $\alpha_p$ \textbf{(ms/tok)} & $\alpha_d$ \textbf{(ms/tok)} & $\beta$ \textbf{(ms)} & $\alpha_d / \alpha_p$ \\
\midrule
InternVL3-1B    & 0.28 & 3.14 & 178 & 11$\times$ \\
InternVL3-2B    & 0.30 & 5.25 & 258 & 17$\times$ \\
Qwen2-VL-2B   & 0.14 & 5.46 & 153 & 39$\times$ \\
Qwen2.5-VL-3B & 0.67 & 8.18 & 143 & 12$\times$ \\
Gemma-3-4B     & 0.74 & 12.46 & 50  & 17$\times$ \\
\bottomrule
\end{tabular}
\end{table}

The roofline model~\cite{williams2009roofline} provides a hardware-level explanation for this asymmetry.
During prefill, the accelerator performs $O(N_{\text{in}})$ multiply-accumulate operations per weight element loaded from memory, producing high \emph{arithmetic intensity} (FLOPs per byte transferred) that saturates the compute units.
During decode, each forward pass reads the full model weights but performs only $O(1)$ useful operations per weight element, resulting in low arithmetic intensity bottlenecked by memory bandwidth.
On the RTX~3070 (256\,GB/s memory bandwidth, 20.3\,TFLOPS FP32), the transition from compute-bound to memory-bound execution produces exactly the order-of-magnitude cost gap we observe.

\begin{table}[t]
\centering
\caption{Prefill and decode time breakdown. $\rho_{\text{decode}}$ denotes the decode fraction of total phase time.}
\label{tab:jetson-timing}
\small
\resizebox{\columnwidth}{!}{%
\begin{tabular}{llccccc}
\toprule
\textbf{Factor} & \textbf{Condition} & $N_{\text{vis}}$ & $\bar{N}_{\text{out}}$ & $t_{\text{prefill}}$ \textbf{(s)} & $t_{\text{decode}}$ \textbf{(s)} & $\rho_{\text{decode}}$ \\
\midrule
\multicolumn{7}{l}{\textit{InternVL3-1B (Q8), Describe prompt}} \\
\midrule
Resolution & 224$\times$224 & 265 & 271 & 1.8 & 11.8 & 0.87 \\
           & 448$\times$448 & 265 & 398 & 1.8 & 18.0 & 0.91 \\
           & 672$\times$672 & 265 & 371 & 1.8 & 16.6 & 0.90 \\
           & 896$\times$896 & 265 & 371 & 2.0 & 17.3 & 0.90 \\
\cmidrule{2-7}
Content    & 1--2 objects    & 265 & 323 & 1.8 & 14.4 & 0.89 \\
           & 3--5 objects    & 265 & 295 & 1.9 & 13.4 & 0.88 \\
           & 6--10 objects   & 265 & 392 & 1.8 & 17.5 & 0.91 \\
           & 11+ objects     & 265 & 413 & 1.9 & 19.3 & 0.91 \\
\cmidrule{2-7}
Prompt     & Describe         & 265 & 398 & 1.8 & 18.0 & 0.91 \\
           & Short answer     & 265 & 3   & 1.8 & 0.07 & 0.04 \\
\midrule
\multicolumn{7}{l}{\textit{Qwen2.5-VL-3B (Q4), Describe prompt}} \\
\midrule
Resolution & 224$\times$224 & 73   & 146 & 1.6  & 23.5 & 0.94 \\
           & 448$\times$448 & 265  & 102 & 6.3  & 16.6 & 0.72 \\
           & 672$\times$672 & 585  & 107 & 15.8 & 17.7 & 0.53 \\
           & 896$\times$896 & 1033 & 109 & 18.7 & 15.8 & 0.46 \\
\cmidrule{2-7}
Content    & 1--2 objects    & mixed & 107 & 11.0 & 16.5 & 0.60 \\
           & 3--5 objects    & mixed & 95  & 10.2 & 15.0 & 0.60 \\
           & 6--10 objects   & mixed & 145 & 10.3 & 23.3 & 0.69 \\
           & 11+ objects     & mixed & 120 & 10.7 & 19.3 & 0.67 \\
\cmidrule{2-7}
Prompt     & Describe         & 265  & 102 & 6.3  & 16.6 & 0.72 \\
           & Short answer     & 265  & 3   & 5.7  & 0.2  & 0.04 \\
\bottomrule
\end{tabular}
}
\end{table}

\subsection{What Determines the Time Split?}
\label{sec:time:split}
 
The asymmetry $\alpha_d \gg \alpha_p$ establishes that decode is inherently more expensive per token, but the actual fraction of wall time spent in each phase depends on three factors that vary across inference instances: input resolution, image content, and prompt type.
We use the Jetson Orin NX to disentangle these factors, as its inference server reports per-phase timings (\texttt{prompt\_ms} and \texttt{predicted\_ms}) that enable direct measurement of the prefill-decode split.
Table~\ref{tab:jetson-timing} summarizes the results, varying one factor at a time while holding the others constant.
 
Resolution affects prefill, not decode.
For InternVL3-1B, which uses a fixed visual tokenizer (259 tokens at all resolutions), prefill remains nearly constant at 1.8\,s regardless of input resolution (CV\,=\,8.9\%), and decode dominates at 87\% to 91\% of total phase time.
For Qwen2.5-VL-3B, which scales visual tokens with resolution, the picture is different: prefill grows from 1.6\,s at 224$\times$224 to 18.7\,s at 896$\times$896, while decode time stays comparable across resolutions.
This produces a \emph{resolution trap}: at 896$\times$896, prefill and decode reach approximate parity (55\% vs.\ 45\%), even though the model ingests $14\times$ more visual tokens but generates approximately the same number of output tokens.
The additional prefill computation is largely wasted from the perspective of response quality.
 
Image content affects decode, not prefill.
Within a fixed resolution, different images produce dramatically different decode times while prefill stays constant.
For InternVL3-1B, prefill is stable at 1.8 to 1.9\,s across all content groups, while decode ranges from 13.4\,s for simple images (3 to 5 objects) to 19.3\,s for complex scenes (11+ objects).
This variation arises entirely because more complex images elicit longer model responses (295 tokens for simple vs.\ 413 for complex), not because of any change in visual processing cost.
The same pattern holds for Qwen2.5-VL-3B: prefill is constant across content groups (10.2 to 11.0\,s when pooled over resolutions), while decode varies from 15.0\,s to 23.3\,s.
The Pearson correlation between output token count and decode time is $r = 0.985$ for InternVL3-1B and $r = 0.980$ for Qwen2.5-VL-3B.
 
Prompt type controls output length directly.
The most striking demonstration comes from switching the prompt.
When we replace ``Describe this image'' (mean 398 output tokens) with a short-answer prompt ``What is the main object? Answer in one word'' (mean 3 output tokens), decode time for InternVL3-1B collapses from 18.0\,s to 0.07\,s, a $250\times$ reduction, while prefill remains unchanged at 1.8\,s.
For Qwen2.5-VL-3B the effect is equally dramatic: decode drops from 16.6\,s to 0.2\,s while prefill stays at 6.3\,s vs.\ 5.7\,s.

These three observations converge on a single conclusion: prefill time is determined by architectural choices (visual token strategy and input resolution), while decode time is determined by output length.
Because $\alpha_d \gg \alpha_p$, decode dominates wall time whenever the model generates a non-trivial response, making output token count the primary determinant of inference duration.
 
\subsection{Output Length Is the Dominant Time Driver}
\label{sec:time:output}

The preceding analysis establishes that decode dominates inference time and that decode time scales linearly with $N_{\text{out}}$.
It follows that inference time is primarily determined by the number of output tokens generated.
For InternVL3-1B on the Jetson, the Pearson correlation between output token count and wall-clock time across all 116 Describe inferences is $r = 0.981$ ($p < 10^{-20}$), and between output token count and decode time specifically is $r = 0.985$.

This finding connects directly to energy through the constant-power result of Section~\ref{sec:power}.
Since $E = \bar{P} \times t$ and $\bar{P}$ is approximately constant, we can write:
\begin{equation}
\label{eq:energy-nout}
E \;\approx\; \bar{P} \cdot (\alpha_d \cdot N_{\text{out}} + c),
\end{equation}
where $c$ absorbs prefill time and fixed overhead.
This equation establishes output token count as the first-order predictor of per-inference energy, and reveals that \emph{reducing energy is equivalent to reducing decode time, which is equivalent to controlling output length}.
The question then becomes: what determines $N_{\text{out}}$, and what are the most effective strategies for controlling inference energy?
We address these questions in the next section.

\section{From Time to Energy}
\label{sec:energy}

Equation~\ref{eq:energy-nout} established that inference energy is a linear function of output token count, with prefill and overhead absorbed into a constant.
In this section, we develop the practical implications of this result: we quantify the energy split between prefill and decode in joules, demonstrate that visual token pruning is far less effective than output length control for reducing energy, and validate the decomposition with a simple predictive model.

\subsection{Energy Decomposition: Prefill vs.\ Decode}
\label{sec:energy:decomp}

Since power is approximately constant throughout inference (Section~\ref{sec:power}), total energy decomposes naturally into prefill and decode contributions:
\begin{equation}
\label{eq:energy-decomp}
E \;=\; \underbrace{\bar{P} \cdot t_{\text{prefill}}}_{E_{\text{prefill}}} \;+\; \underbrace{\bar{P} \cdot t_{\text{decode}}}_{E_{\text{decode}}}.
\end{equation}
Building on the latency model from Section~\ref{sec:time}, we define the \emph{Output Length Sensitivity} (OLS) as the marginal energy cost of generating one additional output token:
\begin{equation}
\label{eq:ols}
\text{OLS} \;=\; \bar{P} \cdot \alpha_d \qquad \text{(J/token)},
\end{equation}
which combines the power fingerprint from Section~\ref{sec:power} with the per-token decode cost from Section~\ref{sec:time}.

Table~\ref{tab:energy-breakdown} reports the energy split for two representative models on the Jetson Orin NX.
For InternVL3-1B (fixed visual tokens), decode accounts for 87\% to 91\% of total energy across all resolutions, with OLS\,=\,0.68\,J/token.
For Qwen2.5-VL-3B (dynamic visual tokens), the prefill energy fraction grows from 6\% at 224$\times$224 to 54\% at 896$\times$896, reflecting the increasing cost of processing more visual tokens.
Even so, at the commonly used 448$\times$448 resolution, decode still accounts for 72\% of total energy.

\begin{table}[t]
\centering
\caption{Energy decomposition between prefill and decode phases.}
\label{tab:energy-breakdown}
\small
\begin{tabular}{llcccc}
\toprule
\textbf{Model} & \textbf{Resolution} & $E_{\text{tot}}$ & $E_{\text{pf}}$ & $E_{\text{dc}}$ & \textbf{Dc\%} \\
\midrule
InternVL3-1B & 224$\times$224 & 192 & 25 & 166 & 87 \\
             & 448$\times$448 & 291 & 27 & 264 & 91 \\
             & 672$\times$672 & 263 & 26 & 236 & 90 \\
             & 896$\times$896 & 274 & 28 & 245 & 90 \\
\midrule
Qwen2.5-VL-3B & 224$\times$224 & 358 & 23  & 338 & 94 \\
               & 448$\times$448 & 324 & 89  & 234 & 72 \\
               & 672$\times$672 & 475 & 223 & 249 & 53 \\
               & 896$\times$896 & 492 & 265 & 223 & 46 \\
\bottomrule
\end{tabular}
\end{table}

\subsection{The Visual Token Pruning Illusion}
\label{sec:energy:pruning}

A growing body of work targets visual token reduction as the primary strategy for efficient VLM inference~\cite{chen2024fastv, yang2025visionzip, xing2025pyramiddrop, shang2025prumerge}.
These methods prune or merge visual tokens to reduce prefill computation, reporting savings in FLOPs or latency.
However, our energy decomposition reveals that this strategy has a fundamental ceiling: \emph{visual token pruning can only reduce prefill energy, which is already a minority of total energy}.

We can bound the maximum energy savings from any visual token pruning method without running it.
Let $\eta$ denote the fraction of visual tokens removed.
The prefill energy reduction is at most $\eta \cdot E_{\text{prefill}}$, since not all prefill computation involves visual tokens (text tokens and attention overhead remain).
The total energy savings under fixed decoding behavior are therefore bounded by:
\begin{equation}
\label{eq:pruning-bound}
\frac{\Delta E}{E} \;\leq\; \eta \cdot \frac{E_{\text{prefill}}}{E_{\text{total}}}.
\end{equation}

Table~\ref{tab:pruning-comparison} applies this bound across four models spanning 1B to 8B parameters and compares it to the energy savings achievable through output length control.
For InternVL3-1B, even removing \emph{all} visual tokens ($\eta = 1$) saves at most 10\% of total energy, while reducing output length by 50\% saves approximately 45\%.
The gap widens at larger scale: for Qwen2.5-VL-7B and InternVL3-8B, complete visual token removal saves only 3--4\%, while a short-answer prompt reduces energy by 96--97\%.
This pattern holds because larger models have higher per-token decode cost $\alpha_d$, which further concentrates energy in the decode phase and diminishes the relative weight of prefill.
Output length control is consistently 2 to 24$\times$ more effective than visual token pruning for reducing energy.

\begin{table}[t]
\centering
\caption{Energy reduction from visual token pruning and output length control.}
\label{tab:pruning-comparison}

\scriptsize
\setlength{\tabcolsep}{2pt}

\begin{tabular}{lcccc}
\toprule
\textbf{Strategy} &
\textbf{InternVL3-1B} & 
\textbf{Qwen2.5-VL-3B} & 
\textbf{Qwen2.5-VL-7B} & 
\textbf{InternVL3-8B} \\
\midrule
Remove 50\% visual tokens
& $\leq$5\%$\downarrow$
& $\leq$14\%$\downarrow$
& $\leq$2\%$\downarrow$
& $\leq$2\%$\downarrow$ \\

Remove 100\% visual tokens
& $\leq$10\%$\downarrow$
& $\leq$28\%$\downarrow$
& $\leq$4\%$\downarrow$
& $\leq$3\%$\downarrow$ \\

\midrule

\texttt{max\_tokens} 256$\rightarrow$128
& $\approx$45\%$\downarrow$
& $\approx$36\%$\downarrow$
& $\approx$48\%$\downarrow$
& $\approx$48\%$\downarrow$ \\

Short-answer prompt (avg. 3 tok)
& 90\%$\downarrow$
& 73\%$\downarrow$
& 96\%$\downarrow$
& 97\%$\downarrow$ \\

\bottomrule
\end{tabular}
\end{table}

The prompt ablation provides direct experimental evidence.
Switching from ``Describe this image'' to a short-answer prompt reduces output from 398 to 3 tokens for InternVL3-1B, cutting energy from 291\,J to 28\,J (90\% reduction) while leaving prefill energy unchanged.
This confirms that \emph{controlling the output side of inference is far more effective than optimizing the input side}.


\subsection{Energy Prediction}
\label{sec:energy:prediction}

\begin{figure}[t!]
    \centering
    \small
    \begin{minipage}[b]{0.48\linewidth}
        \centering
        \includegraphics[width=\linewidth]{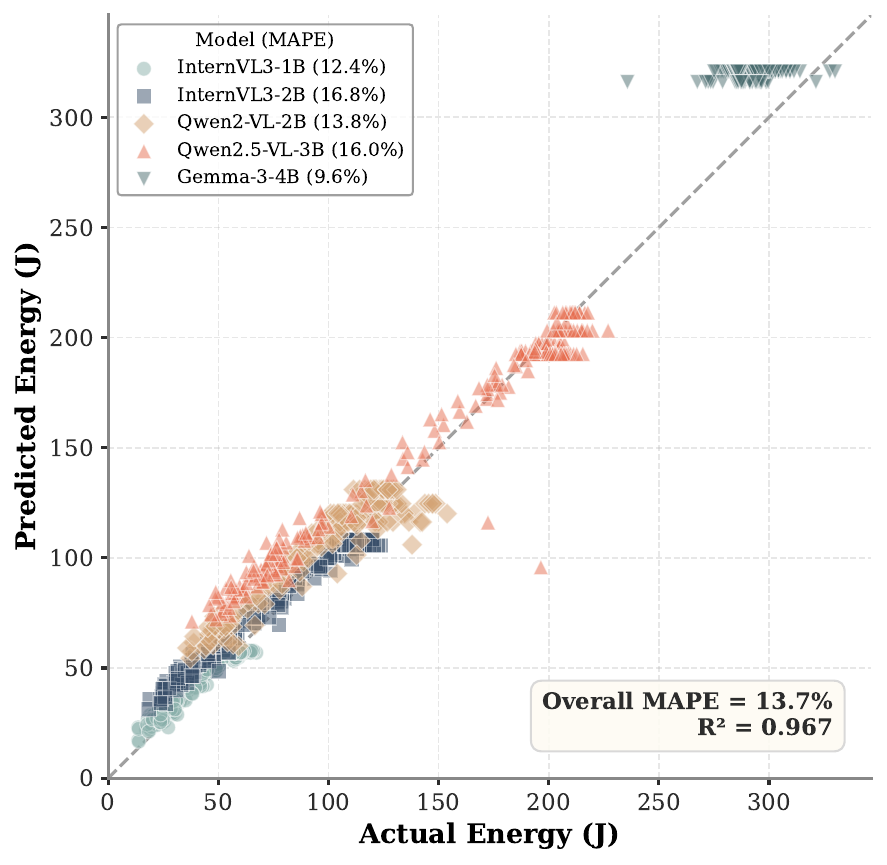}
        \centerline{(a) Per-model predictor}
    \end{minipage}
    \hfill
    \begin{minipage}[b]{0.48\linewidth}
        \centering
        \includegraphics[width=\linewidth]{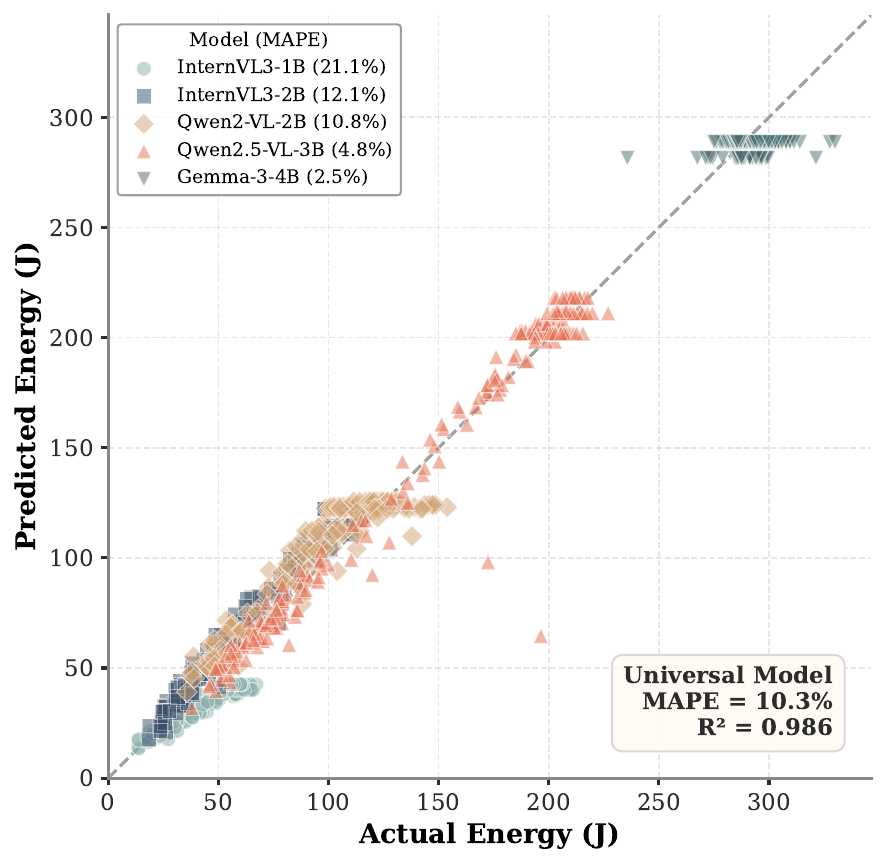}
        \centerline{(b) Universal predictor}
    \end{minipage}
    \caption{Predicted versus measured energy using per-model and universal predictors.}
    \label{fig:energy_pred}
\end{figure}

The decomposition $E = \bar{P} \cdot (\alpha_p N_{\text{in}} + \alpha_d N_{\text{out}} + \beta)$ implies that energy can be predicted from a small number of architectural parameters. We validate this with two linear predictors of increasing generality.

\textit{Per-model predictor.}
For each model, we fit:
\begin{equation}
\label{eq:permodel}
\hat{E} = \bar{P} \cdot (\alpha_p \cdot N_{\text{in}} + \alpha_d \cdot N_{\text{out}} + \beta),
\end{equation}
where $\bar{P}$, $\alpha_p$, $\alpha_d$, and $\beta$ are fitted from a calibration set.
Given a new input, the predictor requires only $N_{\text{in}}$ (deterministic from architecture and resolution) and $N_{\text{out}}$ (bounded by \texttt{max\_tokens}).
Evaluated across all 1{,}680 RTX~3070 runs, the overall MAPE is 13.7\% with $R^2 = 0.967$ (Figure~\ref{fig:energy_pred}a).

\textit{Universal predictor.}
A more ambitious goal is predicting energy for \emph{any} model without per-model calibration.
We fit a single linear model across all five models using interaction features:
\begin{equation}
\label{eq:universal}
\hat{E} = w_0 + w_1 \cdot S + w_2 \cdot N_{\text{in}} + w_3 \cdot N_{\text{out}} + w_4 \cdot S \cdot N_{\text{in}} + w_5 \cdot S \cdot N_{\text{out}},
\end{equation}
where $S$ is the model size in billions of parameters.
The interaction terms $S \cdot N_{\text{in}}$ and $S \cdot N_{\text{out}}$ capture the insight from Section~\ref{sec:time} that larger models incur higher per-token costs.
This predictor achieves MAPE\,=\,10.3\% and $R^2$\,=\,0.986 across all 1{,}680 runs \emph{without any per-model fitting} (Figure~\ref{fig:energy_pred}b).
The dominant coefficient is $w_5$ ($S \times N_{\text{out}}$\,=\,0.426), confirming that the interaction between model size and output length is the primary energy driver.

Leave-one-resolution-out cross-validation shows stable accuracy (MAPE 9.7\% to 12.0\%), indicating that the predictor generalizes to unseen resolution configurations (Table~\ref{tab:loocv}).
The success of such a simple model validates our decomposition framework: VLM inference energy is low-dimensional, fully determined by model size, input token count, and output token count.

\begin{table}[t!]
\centering
\caption{Leave-one-resolution-out evaluation of the universal energy predictor.}
\label{tab:loocv}
\small
\begin{tabular}{lc}
\toprule
\textbf{Held-out Resolution} & \textbf{MAPE (\%)} \\
\midrule
224$\times$224 & 12.0 \\
448$\times$448 & 11.4 \\
672$\times$672 & 9.7 \\
896$\times$896 & 10.3 \\
\midrule
Average & 10.9 \\
\bottomrule
\end{tabular}
\end{table}

\subsection{Deployment Guidelines}
\label{sec:energy:guidelines}

The energy analysis yields three actionable guidelines for energy-constrained VLM deployment.
(i) Budget output, not input.
Since each output token costs 11 to 39$\times$ more than each input token (Table~\ref{tab:alpha-beta}), setting a lower \texttt{max\_tokens} is the single most effective energy control.
Reducing maximum output from 256 to 128 tokens saves approximately 40 to 50\% energy across all models, while resolution reduction from 896$\times$896 to 224$\times$224 saves at most 27\% and only for dynamic-token models.
(ii) Match token strategy to deployment scenario.
If the application requires high-resolution inputs (e.g., fine-grained visual inspection), fixed-token architectures (InternVL3, Gemma-3) avoid the resolution-dependent energy penalty of dynamic-token models.
If resolution is fixed and low, dynamic-token models (Qwen2-VL, Qwen2.5-VL) can be competitive due to their fewer visual tokens at low resolution.
(iii) Anticipate content-driven energy variance.
Image content causes up to 4.1$\times$ energy variation through output length (104 to 428\,J for InternVL3-1B at 448$\times$448).
For battery budgeting in embodied agents or mobile applications, worst-case energy estimation should use \texttt{max\_tokens} rather than average output length.
Our universal predictor (Equation~\ref{eq:universal}) enables estimation without per-configuration profiling.

\section{Related Work}
Our work bridges three areas: edge LLM energy profiling, VLM token efficiency, and efficient edge VLM architectures. Prior work has not characterized the energy consumption of multimodal VLM inference on edge hardware.

\subsection{Energy Profiling for Edge LLM Inference}

A growing body of work characterizes the energy consumption of language model inference on resource-constrained hardware~\cite{zhan2026prism}.
Husom et al.~\cite{husom2025sustainable} benchmark 28 quantized SLMs on Raspberry Pi using a hardware power meter at 2\,MHz sampling, reporting up to 69\% energy reduction from quantization and a correlation of $r{=}0.85$ between response length and total energy.
Camel~\cite{xu2025camel} formulates GPU frequency and batch size selection as a Thompson Sampling bandit on Jetson AGX Orin, achieving 12\% to 30\% energy-delay product reduction for text-only LLMs.
Arya and Simmhan~\cite{arya2025understanding} profile LLMs across eight power modes on the same platform, revealing that quantization can paradoxically increase latency for smaller models.
At the component level, CLEAR~\cite{jain2025clear} addresses the temporal mismatch between microsecond execution and millisecond sensor resolution, finding that attention blocks consume significantly more energy per FLOP than feed-forward layers.
Recent tools such as EdgeProfiler~\cite{pinnock2025edgeprofiler} and Edge-First LM~\cite{jang2025edgefirst} provide analytical models and benchmarks for Qwen~2.5 variants across multiple Jetson platforms, while Abstreiter et al.~\cite{abstreiter2025painful} show that edge inference can be up to 375$\times$ cheaper per token than cloud-based alternatives.
Despite this progress, all of these works study text-only models exclusively.
None extends to multimodal VLM inference, leaving the energy impact of vision encoding, visual token processing, and the interaction between input modality and output length entirely uncharacterized.

\subsection{Visual Token Efficiency in VLMs}

Significant progress has been made in reducing the computational cost of visual tokens in VLMs.
FastV~\cite{chen2024fastv} prunes 50\% of visual tokens after the first two LLM layers with negligible accuracy loss.
LLaVA-PruMerge~\cite{shang2025prumerge} exploits attention sparsity to achieve 14 to 18$\times$ token compression, and PyramidDrop~\cite{xing2025pyramiddrop} leverages visual redundancy in deeper layers to reduce FLOPs by 55\%.
Other approaches, including HiRED~\cite{arif2025hired}, VisionZip~\cite{yang2025visionzip}, SparseVLM~\cite{zhang2025sparsevlm}, and DivPrune~\cite{alvar2025divprune}, introduce attention-guided or text-aware sparsification strategies.
A comprehensive survey by Shao et al.~\cite{shao2025survey} provides an overview of this rapidly growing area.

However, these works evaluate efficiency exclusively through FLOPs, latency, or memory; none reports actual energy consumption in Joules or Watts.
Whether visual token reduction translates to proportional energy savings on edge hardware remains an open question.
As we show in Section~\ref{sec:energy:pruning}, the answer depends critically on the energy split between prefill and decode, which cannot be inferred from FLOPs alone.

\subsection{Efficient Edge VLMs}
Beyond token-level optimizations, researchers have also developed compact VLM architectures specifically for on-device inference.
MobileVLM~\cite{chu2024mobilevlm} reduces visual tokens to 144 and achieves 21.5 tokens/sec on mobile SoCs.
FastVLM~\cite{vasu2025fastvlm} employs a hybrid CNN-ViT encoder for faster time-to-first-token, while SmolVLM~\cite{marafioti2025smolvlm} pushes model size down to 256M parameters, requiring less than 1\,GB of GPU memory.
Other notable efforts include MiniCPM-V~\cite{yao2024minicpmv}, TinyLLaVA~\cite{zhou2024tinyllava}, AndesVL~\cite{jin2025andesvl}, and HyperVL~\cite{liu2025hypervl}, all targeting mobile or NPU platforms.

None of these works provides empirical energy measurements, even when battery life is cited as a primary motivation.
Our work fills this gap through the first systematic energy profiling of VLM inference on GPU-equipped edge hardware.
We show that the dominant energy cost lies not in the visual processing these architectures optimize, but in the autoregressive output generation they overlook.
Two concurrent works address related problems: Alvarez et al.~\cite{alvarez2025scaling} study energy scaling laws for CPU-only LLM inference, and Chung et al.~\cite{chung2026joules} diagnose inference energy at datacenter scale.
Neither constitutes a systematic VLM energy profiling on edge GPUs.

\section{Conclusion}
\label{sec:conclusion}

This work presents the first systematic energy profiling of VLM inference on edge hardware, spanning five models across three architecture families, four input resolutions, and two hardware platforms (NVIDIA RTX 3070 and Jetson Orin NX).
Our analysis reveals a simple but powerful decomposition: power is a model-intrinsic constant with less than 5\% variation across all conditions, inference time is dominated by autoregressive decoding that accounts for 86 to 97\% of total energy, and energy follows directly as their product.
Each output token costs 11 to 39$\times$ more wall-clock time than each input token, and image complexity induces up to 4.1$\times$ energy variation at identical resolution through its effect on output length.
The dominant energy cost in VLM inference is not \emph{seeing} but \emph{speaking}: even removing all visual tokens saves at most 10\% of total energy for fixed-token models, while controlling output length saves up to 97\%.
Validation on 7B and 8B models confirms that this pattern strengthens at larger scale.

Our profiling focuses on 1B--4B parameter models with batch size one, representing a common single-agent edge deployment setting. Whether decode dominance extends to batched or concurrent serving remains an open question.
Looking ahead, our findings suggest shifting VLM efficiency efforts from input-side optimization toward output-side optimization. Existing visual token pruning methods primarily reduce prefill cost, while techniques targeting decode efficiency, such as speculative decoding and output length prediction, directly address the dominant energy phase. Our decomposition also enables energy-aware inference scheduling, such as adapting output budgets to battery constraints. A promising direction is predicting energy-expensive inputs before inference by linking output length with image structural information~\cite{finzi2025epiplexity}.


\begin{acks}
This research was supported by the Youth S\&T Talent Support Programme of Guangdong Provincial Association for Science and Technology (Grant No. SKXRC2025325). We also thank Prof. Giuliano Casale for supporting our ACM SIGMM Student Travel Grant application.
\end{acks}

\bibliographystyle{ACM-Reference-Format}
\bibliography{reference} 

@String{Computing = "Computing" }

@String{Computer = "{IEEE} Computer" }

@String{Chelsea = "Chelsea" }

@inproceedings{brohan2023rt2,
  title     = {{RT-2}: Vision-Language-Action Models Transfer Web Knowledge to Robotic Control},
  author    = {Brohan, Anthony and Brown, Noah and Carbajal, Justice and Chebotar, Yevgen and Chen, Xi and Choromanski, Krzysztof and Ding, Tianli and Driess, Danny and others},
  booktitle = {Proceedings of the 7th Conference on Robot Learning (CoRL)},
  year      = {2023},
}

@inproceedings{driess2023palme,
  title     = {{PaLM-E}: An Embodied Multimodal Language Model},
  author    = {Driess, Danny and Xia, Fei and Sajjadi, Mehdi S. M. and Lynch, Corey and Chowdhery, Aakanksha and Ichter, Brian and Wahid, Ayzaan and Tompson, Jonathan and Vuong, Quan and Yu, Tianhe and others},
  booktitle = {Proceedings of the 40th International Conference on Machine Learning (ICML)},
  year      = {2023},
}

@inproceedings{kim2024openvla,
  title     = {{OpenVLA}: An Open-Source Vision-Language-Action Model},
  author    = {Kim, Moo Jin and Pertsch, Karl and Karamcheti, Siddharth and Xiao, Ted and Balakrishna, Ashwin and Nair, Suraj and Rafailov, Rafael and Foster, Ethan and Lam, Grace and Sanketi, Pannag and others},
  booktitle = {Proceedings of the 8th Conference on Robot Learning (CoRL)},
  year      = {2024},
}

@inproceedings{black2024pi0,
  title     = {$\pi_0$: A Vision-Language-Action Flow Model for General Robot Control},
  author    = {Black, Kevin and Brown, Noah and Driess, Danny and Esmail, Adnan and Equi, Michael and Finn, Chelsea and Fusai, Niccolo and Groom, Lachy and Hausman, Karol and Ichter, Brian and others},
  booktitle = {Proceedings of Robotics: Science and Systems (RSS)},
  year      = {2025},
}

@article{ma2024vlasurvey,
  title   = {A Survey on Vision-Language-Action Models for Embodied {AI}},
  author  = {Ma, Yueen and Song, Zixing and Zhuang, Yuzheng and Hao, Jianye and King, Irwin},
  journal = {arXiv preprint arXiv:2405.14093},
  year    = {2024},
}

@article{chu2024mobilevlm,
  title   = {{MobileVLM V2}: Faster and Stronger Baseline for Vision Language Model},
  author  = {Chu, Xiangxiang and Qiao, Limeng and Zhang, Xinyu and Xu, Shuang and Wei, Fei and others},
  journal = {arXiv preprint arXiv:2402.03766},
  year    = {2024},
}

@inproceedings{vasu2025fastvlm,
  title     = {{FastVLM}: Efficient Vision Encoding for Vision Language Models},
  author    = {Vasu, Pavan Kumar Anasosalu and Faghri, Fartash and Li, Chun-Liang and Koc, Cem and True, Nate and Antony, Albert and Santhanam, Gokul and Gabriel, James and Grasch, Peter and Tuzel, Oncel and Pouransari, Hadi},
  booktitle = {Proceedings of the IEEE/CVF Conference on Computer Vision and Pattern Recognition (CVPR)},
  year      = {2025},
}

@article{marafioti2025smolvlm,
  title   = {{SmolVLM}: Redefining Small and Efficient Multimodal Models},
  author  = {Marafioti, Andr\'{e}s and Zohar, Orr and Farré, Miquel and Noyan, Merve and Bakouch, Elie and others},
  journal = {arXiv preprint arXiv:2504.05299},
  year    = {2025},
}

@inproceedings{zhan2026prism,
  title={Prism: Privacy-aware routing for adaptive cloud--edge llm inference via semantic sketch collaboration},
  author={Zhan, Junfei and Shen, Haoxun and Lin, Zheng and He, Tengjiao},
  booktitle={Proceedings of the AAAI Conference on Artificial Intelligence},
  volume={40},
  number={33},
  pages={28150--28158},
  year={2026}
}

@article{xu2025camel,
  title   = {{Camel}: Energy-Aware {LLM} Inference on Resource-Constrained Devices},
  author  = {Xu, Hao and Peng, Long and Song, Shezheng and Liu, Xiaodong and Jun, Ma and Li, Shasha and Yu, Jie and Mao, Xiaoguang},
  journal = {arXiv preprint arXiv:2508.09173},
  year    = {2025},
}

@article{husom2025sustainable,
  title   = {Sustainable {LLM} Inference for Edge {AI}: Evaluating Quantized {LLMs} for Energy Efficiency, Output Accuracy, and Inference Latency},
  author  = {Husom, Erik Johannes and Goknil, Arda and Astekin, Mustafa and Shar, Lwin Khin and others},
  journal = {ACM Transactions on Internet of Things},
  volume  = {6},
  number  = {4},
  year    = {2025},
}

@inproceedings{arya2025understanding,
  title     = {Understanding the Performance and Power of {LLM} Inferencing on Edge Accelerators},
  author    = {Arya, Mayank and Simmhan, Yogesh},
  booktitle = {7th Workshop on Parallel AI and Systems for the Edge (PAISE), co-located with IEEE IPDPS},
  year      = {2025},
}

@article{jain2025clear,
  title   = {Dissecting Transformers: A {CLEAR} Perspective towards Green {AI}},
  author  = {Jain, Hemang and Goyal, Shailender and Pandey, Divyansh and Vaidhyanathan, Karthik},
  journal = {arXiv preprint arXiv:2510.02810},
  year    = {2025},
}

@inproceedings{chen2024fastv,
  title={An Image is Worth 1/2 Tokens After Layer 2: Plug-and-Play Inference Acceleration for Large Vision-Language Models},
  author={Chen, Liang and Zhao, Haozhe and Liu, Tianyu and Bai, Shuai and Lin, Junyang and Zhou, Chang and Chang, Baobao},
  booktitle={Proceedings of the European Conference on Computer Vision (ECCV)},
  pages={19--35},
  year={2024},
}

@inproceedings{xing2025pyramiddrop,
  title     = {{PyramidDrop}: Accelerating Your Large Vision-Language Models via Pyramid Visual Redundancy Reduction},
  author    = {Xing, Long and Huang, Qidong and Dong, Xiaoyi and Lu, Jiajie and Zhang, Pan and others},
  booktitle = {Proceedings of the IEEE/CVF Conference on Computer Vision and Pattern Recognition (CVPR)},
  year      = {2025},
}

@inproceedings{yang2025visionzip,
  title     = {{VisionZip}: Longer is Better but Not Necessary in Vision Language Models},
  author    = {Yang, Senqiao and Chen, Yukang and Tian, Zhuotao and others},
  booktitle = {Proceedings of the IEEE/CVF Conference on Computer Vision and Pattern Recognition (CVPR)},
  year      = {2025},
}

@article{wang2024qwen2vl,
  title   = {{Qwen2-VL}: Enhancing Vision-Language Model's Perception of the World at Any Resolution},
  author  = {Wang, Peng and Bai, Shuai and Tan, Sinan and Wang, Shijie and Fan, Zhihao and Bai, Jinze and Chen, Keqin and Liu, Xuejing and Wang, Jialin and Ge, Wenbin and others},
  journal = {arXiv preprint arXiv:2409.12191},
  year    = {2024},
}

@article{bai2025qwen25vl,
  title   = {{Qwen2.5-VL} Technical Report},
  author  = {Bai, Shuai and Chen, Keqin and Liu, Xuejing and Wang, Jialin and Ge, Wenbin and Song, Sibo and Dang, Kai and Wang, Peng and Wang, Shijie and Tang, Jun and others},
  journal = {arXiv preprint arXiv:2502.13923},
  year    = {2025},
}

@article{gemmateam2025gemma3,
  title   = {Gemma 3 Technical Report},
  author  = {{Gemma Team} and Kamath, Aishwarya and Ferret, Johan and Pathak, Shreya and Vieillard, Nino and Merhej, Ramona and others},
  journal = {arXiv preprint arXiv:2503.19786},
  year    = {2025},
}

@article{zhu2025internvl3,
  title   = {{InternVL3}: Exploring Advanced Training and Test-Time Recipes for Open-Source Multimodal Models},
  author={Zhu, Jinguo and Wang, Weiyun and Chen, Zhe and Liu, Zhaoyang and Ye, Shenglong and Gu, Lixin and Tian, Hao and Duan, Yuchen and Su, Weijie and Shao, Jie and others},
  journal = {arXiv preprint arXiv:2504.10479},
  year    = {2025},
}

@article{pinnock2025edgeprofiler,
  title   = {{EdgeProfiler}: A Fast Profiling Framework for Lightweight {LLMs} on Edge Using Analytical Model},
  author  = {Pinnock, Alyssa and Jayakody, Shakya and Roxy, Kawsher A. and Ahmed, Md Rubel},
  journal = {arXiv preprint arXiv:2506.09061},
  year    = {2025},
}

@inproceedings{jang2025edgefirst,
  title     = {Edge-First Language Model Inference: Models, Metrics, and Tradeoffs},
  author    = {Jang, SiYoung and Morabito, Roberto},
  booktitle = {Proceedings of the 45th IEEE International Conference on Distributed Computing Systems (ICDCS)},
  year      = {2025},
}

@article{abstreiter2025painful,
  title   = {Sometimes Painful but Certainly Promising: Feasibility and Trade-offs of Language Model Inference at the Edge},
  author  = {Abstreiter, Maximilian and Tarkoma, Sasu and Morabito, Roberto},
  journal = {ACM Transactions on Embedded Computing Systems},
  year    = {2025},
  doi     = {10.1145/3788870},
}

@inproceedings{shang2025prumerge,
  title     = {{LLaVA-PruMerge}: Adaptive Token Reduction for Efficient Large Multimodal Models},
  author    = {Shang, Yuzhang and Cai, Mu and Xu, Bingxin and Lee, Yong Jae and Yan, Yan},
  booktitle = {Proceedings of the IEEE/CVF International Conference on Computer Vision (ICCV)},
  pages     = {22857--22867},
  year      = {2025},
}

@inproceedings{arif2025hired,
  title     = {{HiRED}: Attention-Guided Token Dropping for Efficient Inference of High-Resolution Vision-Language Models},
  author    = {Arif, Kazi Hasan Ibn and Yoon, JinYi and Nikolopoulos, Dimitrios S. and Vandierendonck, Hans and John, Deepu and Ji, Bo},
  booktitle = {Proceedings of the AAAI Conference on Artificial Intelligence},
  volume    = {39},
  number    = {2},
  pages     = {1773--1781},
  year      = {2025},
}

@inproceedings{zhang2025sparsevlm,
  title     = {{SparseVLM}: Visual Token Sparsification for Efficient Vision-Language Model Inference},
  author={Zhang, Yuan and Fan, Chun-Kai and Ma, Junpeng and Zheng, Wenzhao and Huang, Tao and Cheng, Kuan and Gudovskiy, Denis A and Okuno, Tomoyuki and Nakata, Yohei and Keutzer, Kurt and others},
  booktitle = {Proceedings of the 42nd International Conference on Machine Learning (ICML)},
  year      = {2025},
}

@inproceedings{alvar2025divprune,
  title     = {{DivPrune}: Diversity-based Visual Token Pruning for Large Multimodal Models},
  author    = {Alvar, Saeed Ranjbar and Singh, Gursimran and Akbari, Mohammad and Zhang, Yong},
  booktitle = {Proceedings of the IEEE/CVF Conference on Computer Vision and Pattern Recognition (CVPR)},
  pages     = {9392--9401},
  year      = {2025},
}

@article{shao2025survey,
  title   = {A Survey of Token Compression for Efficient Multimodal Large Language Models},
  author  = {Shao, Kele and Tao, Keda and Zhang, Kejia and Feng, Sicheng and Cai, Mu and Shang, Yuzhang and You, Haoxuan and Qin, Can and Sui, Yang and Wang, Huan},
  journal = {arXiv preprint arXiv:2507.20198},
  year    = {2025},
}

@article{yao2024minicpmv,
  title   = {Efficient {GPT-4V} Level Multimodal Large Language Model for Deployment on Edge Devices},
  author  = {Yao, Yuan and Yu, Tianyu and Zhang, Ao and Wang, Chongyi and Cui, Junbo and Zhu, Hongji and Cai, Tianchi and Li, Haoyu and Zhao, Weilin and He, Zhihui and others},
  journal = {Nature Communications},
  volume={16},
  number={1},
  pages={5509},
  year={2025},
  doi     = {10.1038/s41467-025-61040-5},
}

@article{zhou2024tinyllava,
  title   = {{TinyLLaVA}: A Framework of Small-scale Large Multimodal Models},
  author  = {Zhou, Baichuan and Hu, Ying and Weng, Xi and Jia, Junlong and Luo, Jie and Liu, Xien and Wu, Ji and Huang, Lei},
  journal = {arXiv preprint arXiv:2402.14289},
  year    = {2024},
}

@article{jin2025andesvl,
  title   = {{AndesVL} Technical Report: An Efficient Mobile-side Multimodal Large Language Model},
  author  = {Jin, Zhiwei and Song, Xiaohui and Wang, Nan and Liu, Yafei and Li, Chao and others},
  journal = {arXiv preprint arXiv:2510.11496},
  year    = {2025},
}

@article{liu2025hypervl,
  title   = {{HyperVL}: An Efficient and Dynamic Multimodal Large Language Model for Edge Devices},
  author  = {Liu, Yuchen and others},
  journal = {arXiv preprint arXiv:2512.14052},
  year    = {2025},
}

@article{alvarez2025scaling,
  title   = {Scaling Laws for Energy Efficiency of Local {LLMs}},
  author  = {Alvarez, Ander and Genuardi, Alessandro and Sinha, Nilotpal and Tiene, Antonio and Okyay, Mikail and Ryskulov, Bakbergen and Montero, David and Mugel, Samuel and Or\'{u}s, Rom\'{a}n},
  journal = {arXiv preprint arXiv:2512.16531},
  year    = {2025},
}

@article{chung2026joules,
  title   = {Where Do the Joules Go? Diagnosing Inference Energy Consumption},
  author  = {Chung, Jae-Won and Wu, Ruofan and Ma, Jeff J. and Chowdhury, Mosharaf},
  journal = {arXiv preprint arXiv:2601.22076},
  year    = {2026},
}

@inproceedings{dosovitskiy2021vit,
  title     = {An Image is Worth 16x16 Words: Transformers for Image Recognition at Scale},
  author    = {Dosovitskiy, Alexey and Beyer, Lucas and Kolesnikov, Alexander and Weissenborn, Dirk and Zhai, Xiaohua and Unterthiner, Thomas and Dehghani, Mostafa and Minderer, Matthias and Heigold, Georg and Gelly, Sylvain and Uszkoreit, Jakob and Houlsby, Neil},
  booktitle = {Proceedings of the International Conference on Learning Representations (ICLR)},
  year      = {2021},
}

@inproceedings{pope2023efficiently,
  title     = {Efficiently Scaling Transformer Inference},
  author    = {Pope, Reiner and Douglas, Sholto and Chowdhery, Aakanksha and Devlin, Jacob and Bradbury, James and Levskaya, Anselm and Heek, Jonathan and Xiao, Kefan and Agrawal, Shivani and Dean, Jeff},
  booktitle = {Proceedings of Machine Learning and Systems (MLSys)},
  volume    = {5},
  year      = {2023},
}

@article{williams2009roofline,
  title     = {Roofline: An Insightful Visual Performance Model for Multicore Architectures},
  author    = {Williams, Samuel and Waterman, Andrew and Patterson, David A.},
  journal   = {Communications of the ACM},
  volume    = {52},
  number    = {4},
  pages     = {65--76},
  year      = {2009},
  publisher = {ACM},
  doi       = {10.1145/1498765.1498785},
}

@article{finzi2025epiplexity,
  title   = {From Entropy to Epiplexity: Rethinking Information for Computationally Bounded Intelligence},
  author  = {Finzi, Marc and Qiu, Shikai and Jiang, Yiding and Izmailov, Pavel and Kolter, J. Zico and Wilson, Andrew Gordon},
  journal = {arXiv preprint arXiv:2601.03220},
  year    = {2026},
}

@inproceedings{lin2014microsoft,
  title={Microsoft {COCO}: Common Objects in Context},
  author={Lin, Tsung-Yi and Maire, Michael and Belongie, Serge and Bourdev, Lubomir and Girshick, Ross and Hays, James and Perona, Pietro and Ramanan, Deva and Zitnick, C Lawrence and Doll{\'a}r, Piotr},
  booktitle={Proceedings of the European Conference on Computer Vision (ECCV)},
  pages={740--755},
  year={2014},
}

\clearpage
\onecolumn
\appendix

\setcounter{section}{0}
\setcounter{subsection}{0}
\setcounter{figure}{0}
\setcounter{table}{0}
\setcounter{equation}{0}
\setcounter{proposition}{0}
\setcounter{corollary}{0}
\renewcommand{\thesection}{A.\arabic{section}}
\renewcommand{\thesubsection}{A.\arabic{section}.\arabic{subsection}}
\renewcommand{\thetable}{A\arabic{table}}
\renewcommand{\thefigure}{A\arabic{figure}}
\renewcommand{\theequation}{A\arabic{equation}}
\renewcommand{\theHsection}{app.\arabic{section}}
\renewcommand{\theHsubsection}{app.\arabic{section}.\arabic{subsection}}
\renewcommand{\theHtable}{app.\arabic{table}}
\renewcommand{\theHfigure}{app.\arabic{figure}}
\renewcommand{\theHequation}{app.\arabic{equation}}

\begin{tcolorbox}[
  colback=demolinkcolor!6,
  colframe=demolinkcolor,
  boxrule=0.8pt,
  arc=3pt,
  left=10pt, right=10pt, top=8pt, bottom=8pt,
  before skip=8pt, after skip=14pt,
]
\centering
\textbf{\large Project Page with Interactive Energy Calculator}\\[3pt]
A browser-based project page is available at:\\[4pt]
{\large\texttt{\href{https://junfei-z.github.io/vlm/}{https://junfei-z.github.io/vlm/}}}\\[4pt]
\small It walks through the main findings with a guided
narrative, renders the key figures from the main paper, and
includes an \textbf{interactive energy calculator}: readers can
pick a model, drag sliders for input and output token counts,
and see the predicted per-inference energy, latency, and decode
fraction update in real time. The page is fully anonymized and
contains no author or institutional information.
\end{tcolorbox}

\section*{Overview}

This document supplements the main paper with the details
needed to reproduce and scrutinize our measurements, and with
two formal results that promote the main paper's empirical
claims to provable statements.
Appendix~\ref{appx:setup} reports the complete experimental
setup, covering hardware platforms, the power-measurement
chain, software stack, dataset, prompts, and the per-trial
measurement protocol.
Appendix~\ref{appx:robustness} reports robustness checks,
including a resolution sweep under DVFS that confirms the
power fingerprint is not an artifact of clock-locking.
Appendix~\ref{appx:theory} gives the roofline-based mechanistic
explanation of why prefill is compute-bound and decoding is
memory-bound, and proves a formal upper bound on the
per-inference average power.
Appendix~\ref{appx:pruning} explains why prior visual-token
pruning methods report savings that we do not reproduce on
edge devices, and proves a structural upper bound on the
end-to-end energy savings any such method can achieve. The
released artifacts that support every claim above are listed
at the end of Appendix~\ref{appx:pruning}, including a
browser-based project page with an interactive energy
calculator linked at the top of this page.

\section{Experimental Setup}
\label{appx:setup}

This section gathers all configuration details needed to
reproduce the measurements in the main paper, including the
two edge platforms, the power-measurement chain on each, the
software stack and quantization choices, the COCO subset and
prompts, and the per-trial protocol used to compute the
prefill/decode timings and energies.

\subsection{Hardware Platforms and Power-Measurement Chain}
\label{appx:hardware}

We conduct all measurements on two edge platforms that span
the practical deployment spectrum of on-device VLM inference,
from embedded modules to mobile workstations.
Figure~\ref{fig:hardware} shows the physical setup of both
platforms, and Table~\ref{tab:hardware-spec} summarizes their
key specifications. The two platforms differ by roughly an
order of magnitude in both TDP and peak compute, yet share
the same CUDA and \texttt{llama.cpp} software stack, so any
qualitative claim that holds on both platforms is unlikely to
be an artifact of a single device class.

\textbf{Jetson Orin NX 16\,GB.}
The embedded platform is an NVIDIA Jetson Orin NX 16\,GB
module, representative of the compute budget available in
robotics and embodied AI systems. The module is housed in a
third-party aluminum enclosure with a passive heatsink and no
active cooling. We operate the device in its 15\,W power mode
and lock all clocks with \texttt{jetson\_clocks}, matching
the configuration used in the main paper. We continuously
poll \texttt{tegrastats} during every run and verify that no
thermal throttling event is recorded across the full
measurement campaign.

We read module-level power directly from the on-board INA3221
current/voltage monitor, which the NVIDIA L4T kernel exposes
as a sysfs hwmon interface under
\texttt{/sys/bus/i2c/drivers/ina3221/}. The INA3221 is a
three-channel monitor; on this carrier board, we identify the
\texttt{VDD\_IN} channel by scanning the per-channel
\texttt{in*\_label} files at startup, then read voltage in
millivolts from the corresponding \texttt{in*\_input} file
and current in milliamps from the matching
\texttt{curr*\_input} file. Instantaneous power in milliwatts
is computed as $P = V \times I / 1000$. \texttt{VDD\_IN}
reports the total input power to the module, covering CPU,
GPU, memory, and IO, which is the appropriate signal for an
embedded module that has no separable GPU power rail. We
sample \texttt{VDD\_IN} from a background Python thread that
runs in parallel with the inference request: the thread is
launched immediately before the request is dispatched to
\texttt{llama-server}, records (timestamp, power) tuples
throughout the request, and is joined after the response
returns. The sampling interval is fixed at 100\,ms.

\textbf{MECHREVO Titan PLUS Laptop with RTX~3070.}
The high-performance edge platform is a MECHREVO Titan PLUS
17.3-inch gaming laptop\footnote{MECHREVO is a laptop OEM.
The exact model used in this work is the \emph{Titan PLUS} configured
with an Intel Core i7-10875H, 32\,GB DDR4, a 1\,TB Samsung
NVMe SSD, and an NVIDIA RTX~3070 Laptop GPU.} equipped with
an Intel Core i7-10875H CPU (8 cores, 16 threads, 2.30\,GHz
base) and an NVIDIA RTX~3070 Laptop GPU. We lock the GPU
core frequency at 1500\,MHz, matching the configuration used
in the main paper, so that all runs observe a stable compute
ceiling. The laptop operates on AC power throughout.

We read GPU package power through HWiNFO64
Pro\footnote{HWiNFO64 Pro v8.10, \url{https://www.hwinfo.com/}.
We use the Pro edition because the free version does not
support unattended CSV logging at sub-second intervals.}, a
widely used Windows hardware monitor that reads power,
frequency, and temperature directly from the on-board sensors
exposed by the GPU driver and the embedded controller. The
relevant signal for our analysis is the GPU package power
for the RTX~3070 Laptop GPU, which is the same quantity that
\texttt{nvidia-smi} reports as \texttt{power.draw} on desktop
GPUs and which is the closest available analog to ``what the
GPU itself is drawing'' on a laptop. HWiNFO64 Pro is
configured to log this signal to a CSV file at a 100\,ms
sampling interval, the same rate used on the Jetson, together
with timestamps that we later align to the
\texttt{llama-server} request timeline. We deliberately do
not use total system-board power, because that signal would
also include the CPU package, the display, and the storage
subsystem, none of which are part of the inference compute
path under examination.

\textbf{Why two different tools.}
Although a single cross-platform tool such as
\texttt{nvidia-smi} runs on both Windows and Linux, it does
not expose useful power data on Tegra-class devices like the
Jetson Orin NX, which have integrated GPUs and route power
telemetry through the INA3221 instead. Using the
platform-native sensor on each device gives us the most
direct and lowest-latency reading. The two tools write to
identical CSV schemas, so this asymmetry does not propagate
into the downstream analysis pipeline.

\begin{table}[t]
\centering
\caption{Specifications of the two edge platforms used in
this work, including the power-measurement chain on each.}
\label{tab:hardware-spec}
\small
\begin{tabular}{lll}
\toprule
                & Jetson Orin NX 16\,GB       & MECHREVO Titan PLUS \\
\midrule
GPU             & Ampere, 1024 CUDA cores     & RTX~3070 Laptop \\
GPU memory      & shared 16\,GB LPDDR5        & 8\,GB GDDR6 \\
CPU             & 8-core Cortex-A78AE         & i7-10875H, 8C/16T \\
System RAM      & shared with GPU             & 32\,GB DDR4 \\
Storage         & NVMe SSD                    & 1\,TB Samsung NVMe \\
Power mode      & 15\,W, \texttt{jetson\_clocks} & GPU at 1500\,MHz \\
Cooling         & passive heatsink            & active, dual fan \\
OS              & Ubuntu 22.04, JetPack 6.0   & Windows 11 23H2 \\
Power reader    & sysfs INA3221, $V \times I$ & HWiNFO64 Pro v8.10 \\
Power signal    & \texttt{VDD\_IN} (module)   & GPU package power \\
Sampling rate   & 100\,ms                     & 100\,ms \\
\bottomrule
\end{tabular}
\end{table}

\begin{figure}[!t]
\centering
\begin{minipage}[t]{0.46\linewidth}
  \centering
  \includegraphics[width=\linewidth,height=0.55\linewidth,keepaspectratio]{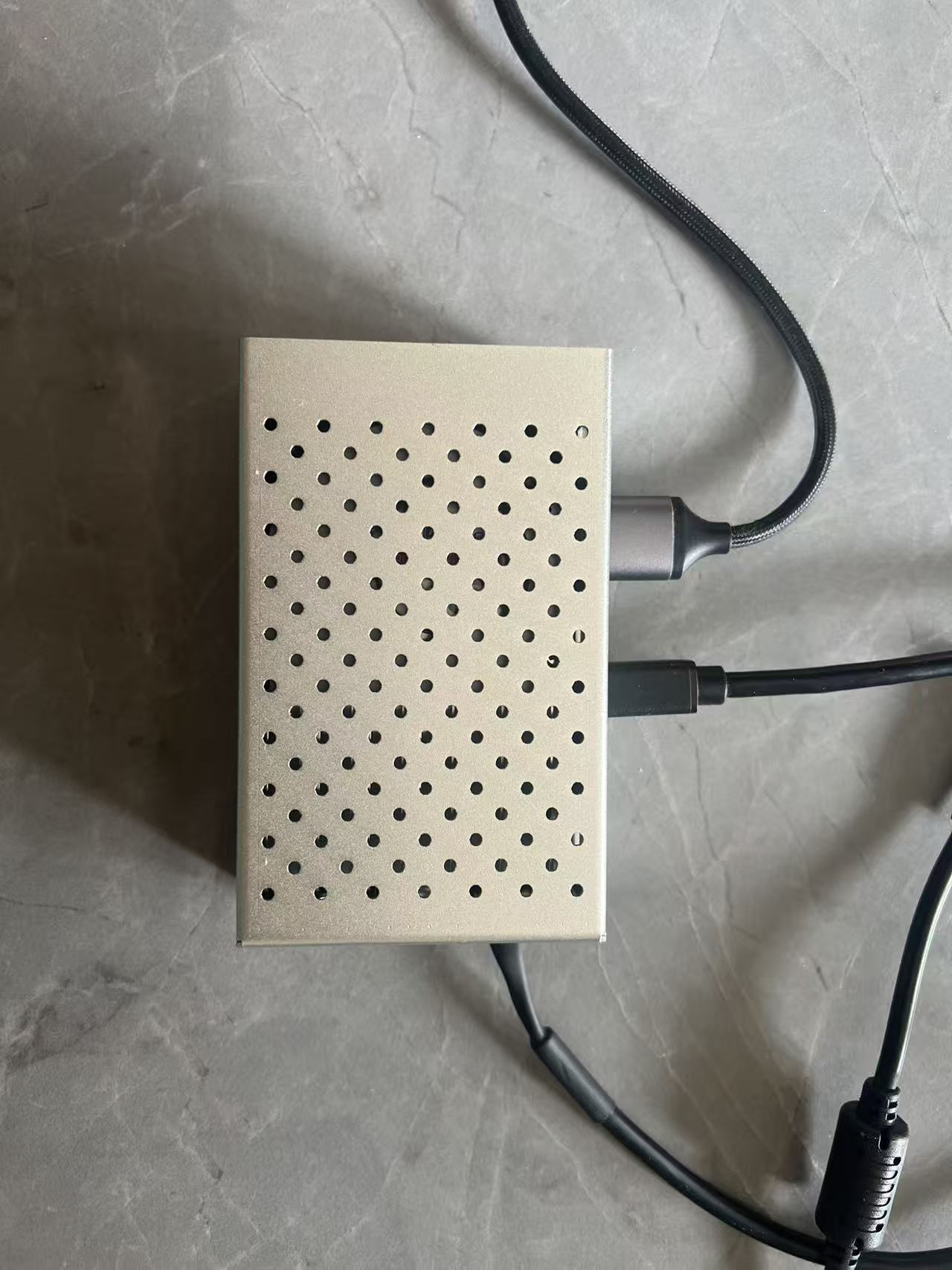}
  \vspace{-1mm}
  \caption*{\small (a) Jetson Orin NX 16\,GB in a passive aluminum enclosure, powered through the USB-C rail monitored by the on-board INA3221 sensor.}
\end{minipage}
\hfill
\begin{minipage}[t]{0.46\linewidth}
  \centering
  \includegraphics[width=\linewidth,height=0.55\linewidth,keepaspectratio]{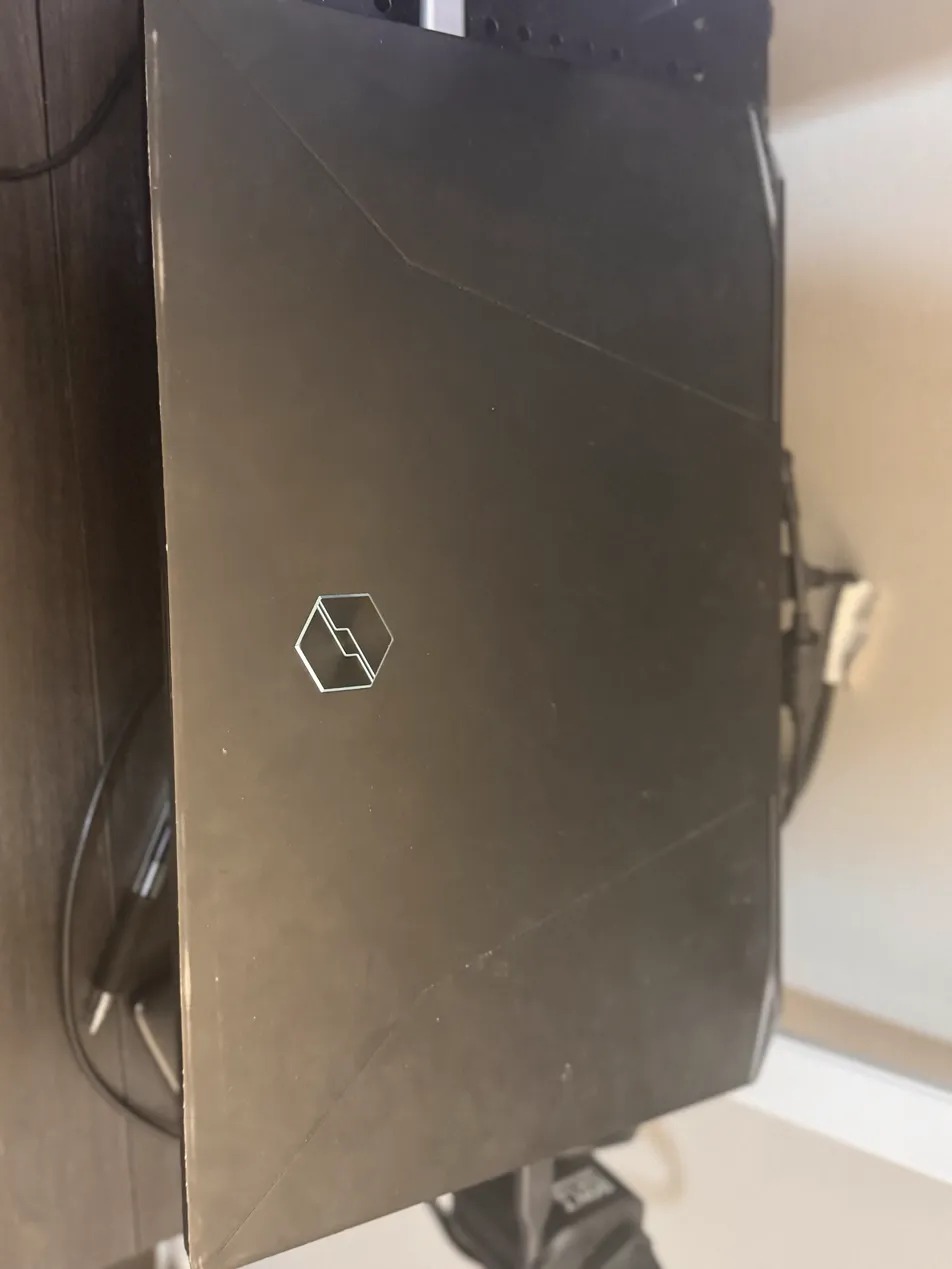}
  \vspace{-1mm}
  \caption*{\small (b) MECHREVO Titan PLUS laptop with an Intel i7-10875H CPU and an RTX~3070 Laptop GPU, frequency-locked at 1500\,MHz.}
\end{minipage}
\vspace{1mm}
\caption{Physical setup of the two edge platforms. Power is
sampled at 100\,ms intervals on both platforms, through the
on-board INA3221 \texttt{VDD\_IN} rail on the Jetson and
through HWiNFO64 Pro on the laptop.}
\label{fig:hardware}
\end{figure}

\subsection{Software Stack and Models}
\label{appx:software}

All inference is performed with \texttt{llama.cpp}, compiled
from source with CUDA support on both platforms. We use
greedy decoding ($T{=}0$) throughout, so sampling variance
does not contaminate the latency and energy measurements.
Every model is served through \texttt{llama-server} with
batch size one, matching the single-agent edge deployment
scenario targeted by the main paper.
Table~\ref{tab:software-spec} lists the exact versions and
serving configuration.

All five models in the main paper (InternVL3-1B,
InternVL3-2B, Qwen2-VL-2B, Qwen2.5-VL-3B, Gemma-3-4B) are
used in their GGUF format, with the language backbone and the
\texttt{mmproj} vision projector loaded as separate files.
We use Q8\_0 quantization for InternVL3-1B and Q4\_K\_M for
the other four, matching Table~1 of the main paper.
Quantization is not re-run by us; we use the released Q-level
files directly to ensure that any follow-up work can
reproduce our results bit-exactly.

\begin{table}[t]
\centering
\caption{Software stack used on both platforms.}
\label{tab:software-spec}
\small
\begin{tabular}{ll}
\toprule
Component                 & Version / configuration \\
\midrule
CUDA                      & 12.2 (laptop and Jetson) \\
Python                    & 3.10 \\
Quantization              & Q8\_0 (InternVL3-1B), Q4\_K\_M (others) \\
KV-cache dtype            & F16 \\
Decoding                  & greedy, $T{=}0$, batch size 1 \\
Server                    & \texttt{llama-server} timings API \\
\bottomrule
\end{tabular}
\end{table}

\subsection{Dataset, Prompts, and Per-Trial Protocol}
\label{appx:protocol}

\textbf{COCO subset.}
We sample 40 images from the COCO~2017 validation
set~\cite{lin2014microsoft}, stratified into six content
complexity tiers T1 through T6, matching the main paper.
Tier assignment is based on the number of COCO-annotated
object instances and distinct categories per image: T1
contains single-object, single-category scenes; T2 contains
2 objects in 1--2 categories; T3, 3--4 objects in 2--3
categories; T4, 5--7 objects in 2--5 categories; T5, 8--11
objects in 3--7 categories; and T6, 12--16 objects in 4--10
categories. Within each tier we pick images that maximize
category diversity and avoid near-duplicate category mixes.
Figure~\ref{fig:dataset-montage} shows four representative
images from each tier. The full image-ID to tier mapping,
together with annotated object and category counts, is
included in the released artifacts. For each image, we cap
the longest side at the target resolution ($224^2$, $448^2$,
$672^2$, or $896^2$) and resize isotropically; aspect ratio
is preserved and no cropping is applied at inference time.

\textbf{Prompts.}
The main paper uses two prompts, reproduced verbatim in
Table~\ref{tab:prompts}. No system prompt is prepended; the
chat template is the default one shipped with each model's
tokenizer, as recognized automatically by
\texttt{llama-server}. The \textbf{Describe} prompt elicits
open-ended captions and exercises the full decode phase; the
\textbf{Short} prompt constrains the model to emit roughly
three output tokens and effectively isolates the prefill
phase.

\begin{table}[t]
\centering
\caption{The two prompts used throughout the main paper.}
\label{tab:prompts}
\small
\begin{tabular}{lp{0.60\linewidth}}
\toprule
Name     & Text \\
\midrule
Describe & ``Describe this image.'' \\
Short    & ``What is the main object in this image? Answer in one word.'' \\
\bottomrule
\end{tabular}
\end{table}

\textbf{Per-trial protocol.}
Each (model, resolution, prompt, image) configuration is run
12 times. The first two runs are discarded as warm-up, and
statistics are computed over the remaining ten; we verified
that two warm-up runs are sufficient to stabilize GPU clocks
and to bring the model weights into the KV-cache allocator.
We define $t_{\text{prefill}}$ and $t_{\text{decode}}$ from
the \texttt{llama-server} timings API, which reports
\texttt{prompt\_ms} and \texttt{predicted\_ms} per request:
$t_{\text{prefill}}$ is the interval from the moment the
input tokens are submitted until the first output token is
emitted (including tokenization and KV-cache initialization),
and $t_{\text{decode}}$ is the interval from the first
output token until the end-of-sequence token. Per-phase
energy is computed by integrating the power trace over these
intervals; the power-sampler clock is aligned to the server
clock at the start of every session by issuing a marker
request. Before each session we record 30\,s of idle power;
the per-inference average power reported in the main paper
is \emph{not} idle-subtracted, so that our power-fingerprint
coefficients directly reflect what a deployment-time power
monitor would observe. For each configuration, we apply a
$3\sigma$ rule on the ten retained trials and re-run any
flagged trial; fewer than 0.5\% of trials triggered this
rule across the full campaign.

\renewcommand{\dbltopfraction}{0.9}
\renewcommand{\dblfloatpagefraction}{0.7}
\begin{figure}[!tb]
\centering
\includegraphics[width=0.8\textwidth]{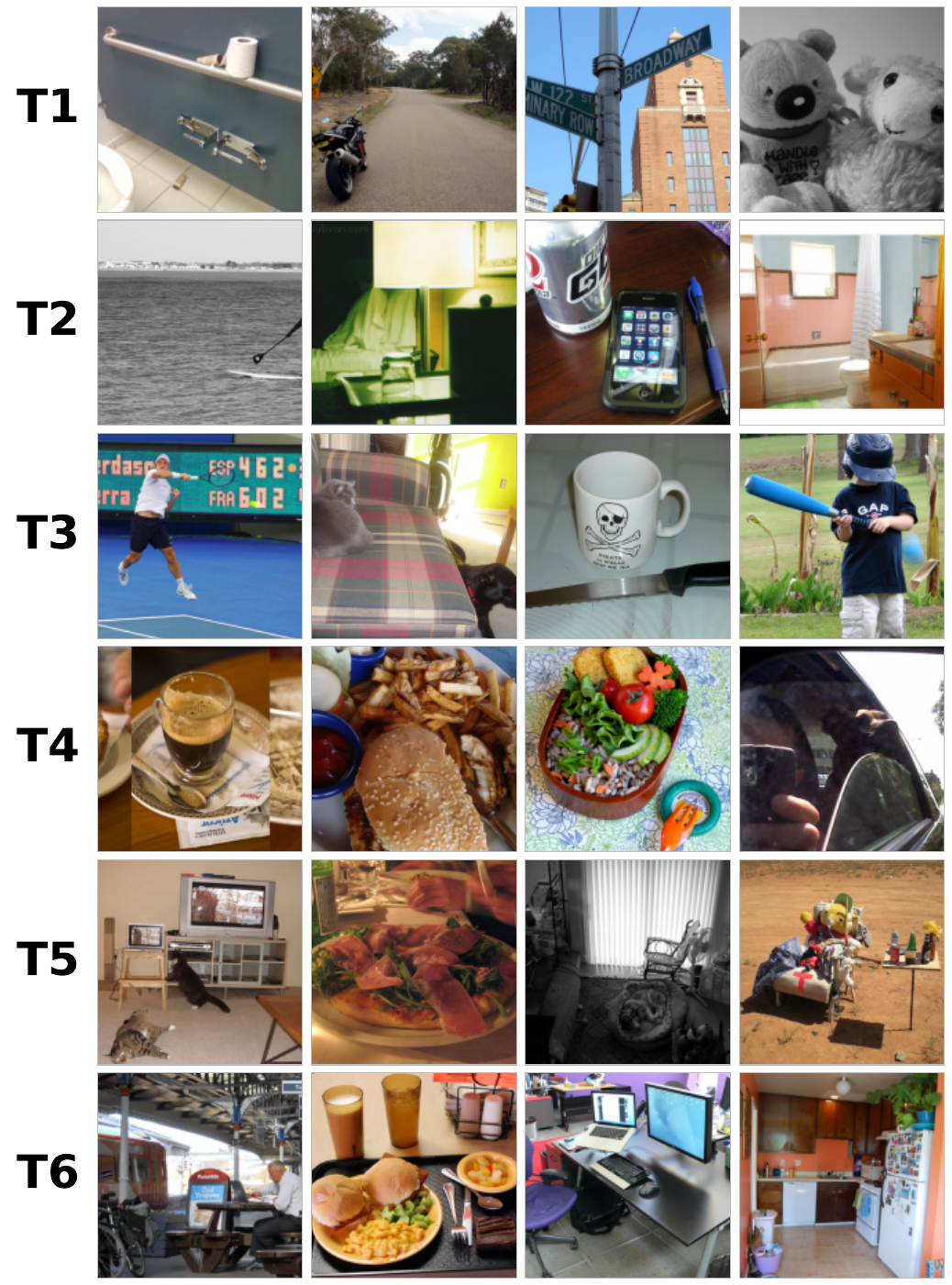}
\caption{Representative images from the 40-image COCO~2017
validation subset used throughout the main paper. Each row
corresponds to one content complexity tier, from T1 (single
object, single category) to T6 (12--16 objects, 4--10
categories). Four images per tier are shown for
illustration. Images are center-cropped to squares for
display only; at inference time we preserve the original
aspect ratio.}
\label{fig:dataset-montage}
\end{figure}

\section{Robustness Checks}
\label{appx:robustness}

This section tests whether the empirical claims of the main
paper depend on incidental choices in our protocol. The most
important check, in Appendix~\ref{appx:dvfs}, repeats the
full RTX~3070 sweep with dynamic voltage and frequency
scaling (DVFS) re-enabled and shows that the power
fingerprint is not an artifact of clock-locking.
Appendix~\ref{appx:other-checks} reports four shorter checks
on trial-level variance, quantization, thermal stability, and
statistical significance.

\subsection{The Power Fingerprint Survives DVFS}
\label{appx:dvfs}

The main paper locks the RTX~3070 GPU core frequency at
1500\,MHz, so that all measurements observe the same compute
ceiling. A natural concern is whether this clock-locking is
itself responsible for the flat power profile reported in
Section~3 of the main paper. In other words: would the power
fingerprint disappear if the GPU were left in its default
DVFS regime, where the driver is free to raise and lower
clocks in response to load?

To answer this, we re-ran the full resolution sweep
($224^2$ to $896^2$, all five models, ten trials per
configuration) on the RTX~3070 with DVFS re-enabled and the
GPU governor restored to its default ``Prefer Maximum
Performance'' profile. Figure~\ref{fig:dvfs} compares the
locked and DVFS results side by side.

The qualitative conclusion of the main paper is preserved.
Panel~(b) shows that under DVFS the per-inference average
power is still essentially flat across resolutions for every
model: each dashed curve is approximately horizontal, just
like its solid (locked) counterpart. The DVFS curves sit
above the locked curves by a model-dependent offset, ranging
from roughly 5--11\% for InternVL3-1B to roughly 23--25\%
for Gemma-3-4B (panel~(d)). This offset is constant across
resolutions for every model, which means that DVFS shifts
the power fingerprint to a higher value but does not break
the invariance: power remains a model-intrinsic constant
under both regimes, only the value of the constant changes.
The energy ratios in panel~(e) tell the same story; the
$896/224$ energy ratio is within $0.92$ to $1.24$ for both
regimes, and the per-token energy at $896^2$ in panel~(f)
preserves the same model ranking under DVFS as under locked
clocks.

We chose the locked-clock setting for the main paper because
it removes one source of run-to-run variance and makes the
power fingerprint claim cleaner to state.
Figure~\ref{fig:dvfs} confirms that this choice does not
manufacture the result: the power fingerprint is a property
of the model and the hardware, not of the clock-locking
protocol.

\subsection{Variance, Quantization, Thermal Stability, and Statistical Significance}
\label{appx:other-checks}

\textbf{Per-configuration variance.}
Every number reported in the main-paper tables is the mean
over ten retained trials. The coefficient of variation on
per-inference average power is below 5\% for every (model,
resolution) configuration on the RTX~3070, and below 3\% for
every (model, content-tier) and (model, prompt) configuration
on the Jetson. This confirms that the invariance claims of
Section~3 of the main paper are not an artifact of averaging
over noisy trials.

\textbf{Quantization.}
The main paper uses Q8\_0 for InternVL3-1B and Q4\_K\_M for
the other four models. To verify that this mix is not
driving the qualitative picture, we additionally measured
InternVL3-1B at Q4\_K\_M and Qwen2.5-VL-3B at Q8\_0 on the
RTX~3070. The absolute per-token energy shifts in the
expected direction (Q4 is lighter, Q8 is heavier), but the
decode-to-prefill cost ratio $\alpha_d/\alpha_p$ remains
within $\pm 10\%$ of the values reported in Table~2 of the
main paper. The qualitative claim that decode dominates
wall-clock time and energy is preserved at both quantization
levels.

\textbf{Thermal stability.}
As noted in Appendix~\ref{appx:hardware}, we verified
absence of thermal throttling events throughout every run on
both platforms. As a secondary check, we compared the first
and last trials of every 12-run batch and found no monotonic
drift in either power or wall-clock time: the
last-minus-first differences are symmetric around zero and
within the measurement noise. This rules out the possibility
that the reported numbers are distorted by cumulative
heating.

\textbf{Statistical significance.}
For each main-paper claim of the form ``factor $X$ does not
affect per-inference average power'' (resolution, content
complexity, prompt type), we ran a one-way ANOVA on the
trial-level power measurements. In every case the
$F$-statistic fails to reject the null hypothesis at the
$\alpha = 0.05$ level after Bonferroni correction across the
three factors. For the correlation between output token
count and decode time, the Pearson coefficient exceeds
$0.98$ for every model on both platforms, with $p < 10^{-20}$
in all cases, consistent with the $r = 0.985$ value reported
in Section~4 of the main paper for InternVL3-1B on the
Jetson.

\renewcommand{\dbltopfraction}{0.9}
\renewcommand{\dblfloatpagefraction}{0.7}
\begin{figure}[!tb]
\centering
\includegraphics[width=0.92\textwidth]{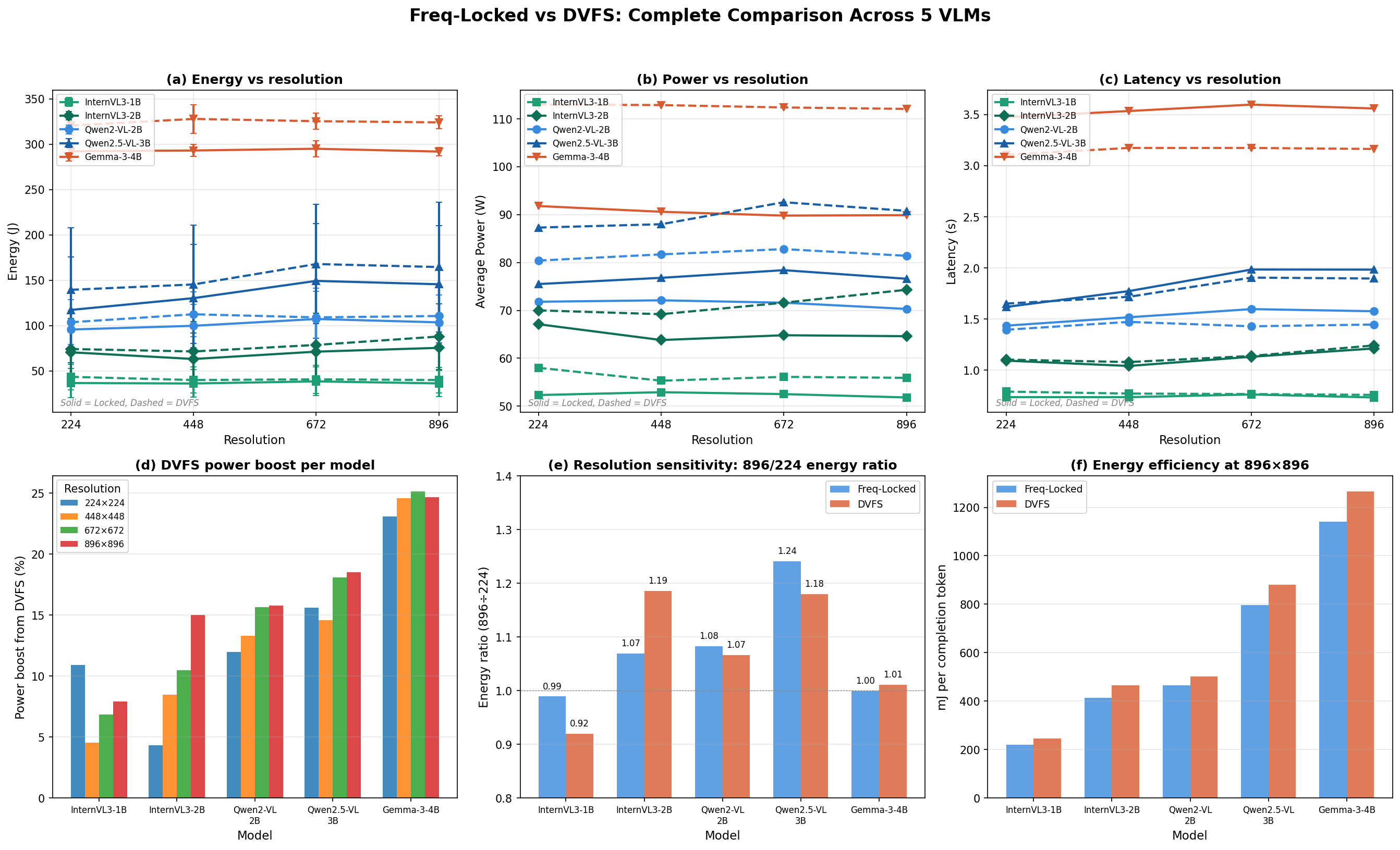}
\caption{Locked-clock (solid) versus DVFS (dashed) comparison
on the RTX~3070 across all five models and four resolutions.
(a)--(c) Per-inference energy, average power, and latency as
a function of input resolution. (d) Per-model power increase
from DVFS, computed at each resolution. (e) Energy ratio
between $896^2$ and $224^2$ inputs under each regime.
(f) Energy per completion token at $896^2$ under each
regime. Across all five models, DVFS shifts the per-inference
average power to a higher value but preserves the
flat-across-resolution profile, confirming that the power
fingerprint claim of Section~3 of the main paper is not an
artifact of clock-locking.}
\label{fig:dvfs}
\end{figure}
\section{Mechanism: Why Power is Constant and Decode Dominates}
\label{appx:theory}

The main paper makes two empirical claims that, taken
together, look surprising: per-inference average power is
constant across inputs, but decoding takes an order of
magnitude longer per token than prefill. This section gives
a unified roofline-based explanation of both observations
and then states the power-constancy claim as a formal
proposition. Both phenomena follow from the same fact, namely
that prefill and decode sit on opposite sides of the
roofline ridge point on every modern GPU.

\subsection{Roofline Decomposition of Prefill and Decode}
\label{appx:roofline}

\textbf{Prefill is compute-bound.}
During prefill, the model processes all $N_{\text{in}}$ input
tokens (text plus visual) in a single forward pass. Each
weight loaded from memory is reused across all
$N_{\text{in}}$ tokens by the attention and feed-forward
matrix multiplications. The arithmetic intensity, measured
in FLOPs per byte, is high and scales linearly with
$N_{\text{in}}$. On both of our platforms this pushes the
workload to the compute-bound side of the roofline, so the
GPU saturates its multiply-accumulate units and runs near
its TDP. The per-token wall time during prefill is
therefore short, and the $\alpha_p$ values reported in
Table~2 of the main paper are correspondingly small.

\textbf{Decode is memory-bound.}
During decode, the model emits one token per forward pass.
Each weight is loaded from memory and used for a single
token before being evicted from cache. The arithmetic
intensity collapses to roughly one multiply-accumulate per
byte loaded. On the RTX~3070 Laptop (256\,GB/s memory
bandwidth, 20.3\,TFLOPS FP32, compute-to-bandwidth ratio of
roughly 79 FLOPs/byte) this intensity sits an order of
magnitude below the roofline ridge point, placing decode
firmly in the memory-bound regime. On the Jetson Orin NX
16\,GB (102\,GB/s LPDDR5 bandwidth, 70 INT8 TOPS at 15\,W)
the ratio is similar and decode is again memory-bound.
Per-token wall time during decode is therefore dominated by
the time it takes to stream the full set of weights from
DRAM, which is exactly why the $\alpha_d$ values in
Table~2 of the main paper are $11$ to $39$ times larger
than $\alpha_p$.

\textbf{Why power stays flat in both regimes.}
It is tempting to assume that a memory-bound phase should
draw less power than a compute-bound phase. Our measurements
show the opposite: per-inference average power is
essentially the same in both regimes. The reason is that
memory transactions themselves consume a substantial
fraction of the GPU power budget. DRAM access, the on-chip
interconnect, and the memory controller all remain highly
active during decode. The compute units are under-utilized
in this regime, but the memory subsystem is fully saturated,
and the two effects roughly cancel at the board level. This
is exactly why the power fingerprint of Section~3 of the
main paper holds across resolutions, contents, and prompts:
changing the input changes which part of the GPU is the
bottleneck, but not the total power envelope.

\subsection{A Formal Bound on the Power Fingerprint}
\label{appx:roofline-prop}

We now state the power fingerprint claim of Section~3 of
the main paper as a formal consequence of the roofline
picture above. Let $\pi_c$ denote the GPU's peak compute
throughput in FLOPs per second, and $\beta_m$ its peak
memory bandwidth in bytes per second. The roofline ridge
point is the arithmetic intensity
$I^{\star} = \pi_c / \beta_m$, measured in FLOPs per byte.
For any workload with arithmetic intensity $I$, the
achievable throughput is $\min(\pi_c, \, I \cdot \beta_m)$.

Let $P_{\max}$ denote the GPU board power when the compute
units are fully saturated, and $P_{\text{mem}}$ the board
power when the memory subsystem is fully saturated. The
qualitative claim of Appendix~\ref{appx:roofline} is that
these two quantities are within a few watts of each other on
modern GPUs, because DRAM access, interconnect, and the
memory controller themselves consume a substantial fraction
of the board power budget. We make this explicit as an
assumption and derive the power fingerprint as its
consequence.

\begin{proposition}[Power fingerprint under roofline]
\label{prop:power-fingerprint}
Suppose that, throughout an inference, the workload is at
all times either compute-bound ($I \geq I^{\star}$) or
memory-bound ($I < I^{\star}$), and that the GPU is fully
utilized in whichever regime it is in. Suppose further that
$P_{\max} \approx P_{\text{mem}}$, with relative gap
$\epsilon = |P_{\max} - P_{\text{mem}}| / P_{\max}$. Then
the per-inference average power $\bar{P}$ satisfies
\begin{equation}
\label{eq:power-bound}
(1 - \epsilon) \, P_{\max} \;\leq\; \bar{P} \;\leq\; P_{\max},
\end{equation}
independently of the input resolution, the image content,
the prompt type, the prefill-to-decode ratio, or the number
of output tokens.
\end{proposition}

\begin{proof}
Let $f$ denote the fraction of inference wall time spent in
the compute-bound regime, and $1 - f$ the fraction spent in
the memory-bound regime. By the saturation assumption, the
instantaneous power is $P_{\max}$ during the compute-bound
fraction and $P_{\text{mem}}$ during the memory-bound
fraction. The time-average power is therefore
\begin{equation}
\bar{P} \;=\; f \, P_{\max} \;+\; (1 - f) \, P_{\text{mem}}.
\end{equation}
Since $0 \leq f \leq 1$ and
$P_{\text{mem}} \in [(1 - \epsilon) P_{\max}, \, P_{\max}]$,
the average lies in the same interval, giving
Equation~\ref{eq:power-bound}. The bound holds for any value
of $f$, hence for any input that changes only the
prefill-to-decode mix without changing $P_{\max}$ or
$P_{\text{mem}}$.
\end{proof}

\textbf{Interpretation.}
Proposition~\ref{prop:power-fingerprint} says that as long
as the GPU stays saturated in whichever roofline regime
applies, and as long as memory-bound and compute-bound power
are close, the per-inference average power is bounded inside
an $\epsilon$-window that does not depend on the input. The
width of this window is exactly the coefficient of variation
we observe empirically: under 5\% on the RTX~3070 and under
3\% on the Jetson Orin NX, which corresponds to
$\epsilon \lesssim 0.05$. The proposition does \emph{not}
claim that $\bar{P}$ is the same across different models,
only that it is the same across different inputs to the same
model on the same hardware. The linear dependence of
$\bar{P}$ on parameter count
($\bar{P} = 12.1\,S + 42.2$ in Equation~1 of the main paper)
reflects how $P_{\max}$ itself scales with model size, not a
violation of the proposition.

\section{An Upper Bound on Visual-Token Pruning Savings}
\label{appx:pruning}

Section~5 of the main paper reports that visual-token
pruning yields end-to-end energy savings between $3\%$ and
$28\%$ on edge devices, while prior work reports savings of
$40\%$ to $60\%$ on the same methods. This section explains
the gap. We first itemize four qualitative sources of
inflation in Appendix~\ref{appx:pruning-gap}, then prove a
structural upper bound in Appendix~\ref{appx:pruning-bound}
that puts a hard ceiling on what any pruning method of this
kind can achieve. The released artifacts that support the
claims of this and the preceding sections are listed in
Appendix~\ref{appx:resources}.

\subsection{Sources of the Reported-vs-Measured Gap}
\label{appx:pruning-gap}

Recent visual-token pruning methods such as
FastV~\cite{chen2024fastv}, VisionZip~\cite{yang2025visionzip},
PyramidDrop~\cite{xing2025pyramiddrop}, and
PruMerge~\cite{shang2025prumerge} report substantial
efficiency gains, often in the range of $40\%$ to $60\%$ on
FLOPs or latency benchmarks. The main paper shows that the
corresponding end-to-end energy savings on edge devices fall
far below these figures. Table~\ref{tab:pruning-gap}
attributes the gap to four reasons that are visible in any
reasonable re-reading of prior work once one stops treating
FLOPs and latency as energy proxies.

\begin{table}[t]
\centering
\caption{Why reported visual-token pruning savings overstate
the actual edge-device energy benefit.}
\label{tab:pruning-gap}
\small
\begin{tabular}{p{0.26\linewidth}p{0.66\linewidth}}
\toprule
Source of gap                & Explanation \\
\midrule
FLOPs used as energy proxy   & Pruning reduces visual-token
                               FLOPs, but FLOPs are a poor
                               proxy for energy when decode
                               dominates total energy, which
                               is exactly the edge regime
                               (main paper, Section~5). \\
Server GPU bias              & Datacenter GPUs have higher
                               compute-to-bandwidth ratios
                               than edge GPUs, so prefill
                               takes a larger share of wall
                               time and pruning looks more
                               impactful than it is on edge
                               silicon. \\
Long-context, short-output benchmarks & Many pruning papers
                               evaluate on benchmarks with
                               very long visual contexts and
                               very short outputs, which
                               inverts the edge-typical ratio
                               and maximally favors
                               prefill-side optimization. \\
No rail-level power meter    & Prior work typically reports
                               latency or FLOPs, not energy
                               measured at the rail.
                               Memory-bound decode consumes
                               real joules even when it does
                               little compute, and only
                               rail-level measurement
                               captures this. \\
\bottomrule
\end{tabular}
\end{table}

\subsection{The Bound and Its Consequences}
\label{appx:pruning-bound}

The four sources of inflation in Table~\ref{tab:pruning-gap}
are qualitative. We now make the central point quantitative
by proving an upper bound on the end-to-end energy savings
any visual-token pruning method can achieve, expressed
entirely in terms of quantities that can be measured before
any pruning is applied.

Let $E_{\text{pre}}$ and $E_{\text{dec}}$ denote the prefill
and decode energy of an unpruned inference, and let
$E_{\text{tot}} = E_{\text{pre}} + E_{\text{dec}}$. Define
$\rho_{\text{pre}} = E_{\text{pre}} / E_{\text{tot}}$ as the
prefill share of the total energy budget. Let
$\eta \in [0, 1]$ denote the fraction of visual tokens
removed by a pruning method, and write $\Delta E$ for the
resulting reduction in $E_{\text{tot}}$.

\begin{proposition}[Energy savings upper bound]
\label{prop:pruning-bound}
Suppose a visual-token pruning method removes a fraction
$\eta$ of visual tokens, leaves the decode phase unchanged,
and reduces prefill energy in proportion to the removed
visual-token compute, that is, by a factor of at most
$\eta$. Then the relative end-to-end energy savings satisfy
\begin{equation}
\label{eq:appx-pruning-bound}
\frac{\Delta E}{E_{\text{tot}}}
\;\leq\;
\eta \cdot \rho_{\text{pre}}
\;\leq\;
\rho_{\text{pre}}.
\end{equation}
The bound is tight in the limit $\eta \to 1$, in which case
the maximum achievable saving is exactly the prefill share
$\rho_{\text{pre}}$.
\end{proposition}

\begin{proof}
By the assumption that pruning leaves decode unchanged, the
post-pruning total energy is
\begin{equation}
E_{\text{tot}}'
\;=\; E_{\text{pre}}' \;+\; E_{\text{dec}}
\;\geq\; (1 - \eta) E_{\text{pre}} \;+\; E_{\text{dec}},
\end{equation}
where the inequality uses the assumption that prefill
energy drops by at most a factor of $\eta$. Subtracting
from $E_{\text{tot}}$ gives
\begin{equation}
\Delta E
\;=\; E_{\text{tot}} - E_{\text{tot}}'
\;\leq\; \eta \, E_{\text{pre}}.
\end{equation}
Dividing both sides by $E_{\text{tot}}$ and substituting
$\rho_{\text{pre}} = E_{\text{pre}} / E_{\text{tot}}$ yields
Equation~\ref{eq:appx-pruning-bound}. The second inequality
follows from $\eta \leq 1$. Setting $\eta = 1$ saturates the
first inequality, giving the tightness claim.
\end{proof}

\begin{corollary}[Edge-device ceiling]
\label{cor:edge-ceiling}
For the five models reported in the main paper, the prefill
share $\rho_{\text{pre}}$ measured on the Jetson Orin NX in
the $15$\,W power mode lies between $0.05$ and $0.30$,
depending on the model and the input resolution. Therefore
the end-to-end energy savings of any visual-token pruning
method on these configurations cannot exceed $30\%$, and
for most (model, resolution) pairs cannot exceed $10\%$.
\end{corollary}

\textbf{What this rules out and what it does not.}
Proposition~\ref{prop:pruning-bound} is a statement about
end-to-end energy on the rail, not about FLOPs or about
prefill latency in isolation. It does not contradict prior
work that reports $50\%$ or higher savings on those proxies;
it only says that those proxies do not transfer to
rail-level joules on edge devices, because the decode share
dwarfs the prefill share. The bound also assumes that
pruning leaves decode unchanged, which is the standard
setting in FastV~\cite{chen2024fastv},
VisionZip~\cite{yang2025visionzip},
PyramidDrop~\cite{xing2025pyramiddrop}, and
PruMerge~\cite{shang2025prumerge}. A method that also
reduced the number of output tokens, for example by
training the model to be more concise, would not be bounded
by Equation~\ref{eq:appx-pruning-bound}, and is in fact the lever
the main paper recommends: controlling output length attacks
the $1 - \rho_{\text{pre}}$ portion of the energy budget
that visual-token pruning cannot touch. The main paper
reports that this lever yields $73\%$ to $97\%$ energy
savings, which is exactly the regime that pruning is
structurally locked out of.

\subsection{Released Artifacts}
\label{appx:resources}

To support reproduction and follow-up work, we release:
\begin{itemize}[leftmargin=*]
  \item \textbf{Project page with interactive energy
        calculator} at\\
        \texttt{\href{https://vlm-energy-profiling-2026.netlify.app/}{https://vlm-energy-profiling-2026.netlify.app/}}.
        The page presents a guided walk-through of our main
        findings, renders the key figures of the main paper,
        and exposes an interactive calculator that takes a
        model, an input token count, and an output token
        count and returns the predicted per-inference energy,
        latency, and decode fraction in real time. The page
        is fully anonymized and already live; no installation
        or account is required.
  \item The measurement and analysis code, including the
        \texttt{llama-server} wrapper, the INA3221 sysfs
        sampler, the HWiNFO64 logging configuration, and
        the phase-alignment logic.
  \item Raw per-trial CSV traces for every (model,
        resolution, prompt, image, trial) configuration
        reported in the main paper, with idle baselines
        recorded alongside.
  \item The COCO subset metadata: image IDs, T1--T6 tier
        labels, annotated object and category counts, and
        the corresponding model outputs (output token count
        and text) for every model.
  \item Build scripts and configuration files for both
        platforms, including the \texttt{jetson\_clocks}
        profile, the laptop GPU frequency-lock script, and
        the HWiNFO64 logging template.
\end{itemize}

The code repository and raw traces will be made public at
the camera-ready stage. The project page linked above is
already live and fully anonymous.

\end{document}